\documentclass[11pt]{article}
\usepackage{lmodern}
    \PassOptionsToPackage{numbers, compress}{natbib}

\usepackage[utf8]{inputenc} 
\usepackage[T1]{fontenc}    

\RequirePackage[colorlinks,citecolor=blue,urlcolor=blue]{hyperref}

\usepackage{url}            
\usepackage{booktabs}       
\usepackage{amsfonts}       
\usepackage{nicefrac}
\usepackage{wrapfig}
\usepackage{microtype}      
\usepackage{amsmath, amsthm} 
\usepackage{amsfonts,amssymb}
\usepackage{algorithm,algorithmic}
\usepackage{nicefrac}       
\usepackage{microtype}      

\usepackage{xcolor}
\usepackage{graphicx}
\usepackage{wrapfig,lipsum}
\usepackage{subcaption}
\usepackage{thmtools, thm-restate}
\usepackage{mathrsfs} 

\usepackage[shortlabels]{enumitem}

\usepackage{bbm}

\definecolor{dgreen}{rgb}{0.00,0.49,0.00}
\definecolor{dblue}{rgb}{0,0.08,0.75}
\newcommand{\cost}{\mathsf{c}}

\newcommand{\sink}{S_\varepsilon}

\newcommand{\oteps}{\textnormal{OT}_\varepsilon}

\newcommand{\dom}{\mathcal{X}}

\newcommand{\eqals}[1]{\begin{align}#1\end{align}}
\newcommand{\eqal}[1]{\begin{align}#1\end{align}}

\newcommand{\bpr}{\begin{proof}}
\newcommand{\epr}{\end{proof}}
\newcommand{\be}{\begin{equation}}
\newcommand{\ee}{\end{equation}}

\newcommand{\bd}{\begin{definition}}
\newcommand{\ed}{\end{definition}}

\newcommand{\bi}{\begin{itemize}}
\newcommand{\ei}{\end{itemize}}

\newcommand{\ba}{\begin{ass}}
\newcommand{\ea}{\end{ass}}

\newcommand{\br}{\begin{remark}}
\newcommand{\er}{\end{remark}}

\newcommand{\bp}{\begin{proposition}}
\newcommand{\ep}{\end{proposition}}

\newcommand{\blm}{\begin{lemma}}
\newcommand{\elm}{\end{lemma}}

\newcommand{\bt}{\begin{theorem}}
\newcommand{\et}{\end{theorem}}

\newcommand{\bcor}{\begin{corollary}}
\newcommand{\ecor}{\end{corollary}}

\newcommand{\bex}{\begin{example}}
\newcommand{\eex}{\end{example}}

\usepackage[capitalize]{cleveref}

\crefname{assumption}{Assumption}{Assumptions}
\crefname{equation}{Eq.}{Eqs.}
\crefname{figure}{Fig.}{Fig.}
\crefname{table}{Table}{Tables}
\crefname{section}{Sec.}{Sec.}
\crefname{theorem}{Thm.}{Thm.}
\crefname{lemma}{Lemma}{Lemmas}
\crefname{corollary}{Cor.}{Cor.}
\crefname{example}{Example}{Examples}
\crefname{remark}{Remark}{Remarks}
\crefname{algorithm}{Alg.}{Algorightms}

\crefname{appendix}{Appendix}{Appendices}

\makeatletter
\def\@endtheorem{\endtrivlist}
\makeatother

\declaretheorem[name=Theorem,refname=Theorem]{theorem}
\declaretheorem[name=Lemma,sibling=theorem]{lemma}

\declaretheorem[name=Proposition,refname=Proposition,sibling=theorem]{proposition}
\declaretheorem[name=Remark,sibling=theorem]{remark}
\declaretheorem[name=Corollary,refname=Corollary,sibling=theorem]{corollary}
\declaretheorem[name=Definition,refname=Definition]{definition}

\declaretheorem[name=Example]{example}

\newcommand{\R}{\ensuremath \mathbb{R}}

\renewcommand{\P}{{\mathcal{P}}}

\DeclareMathOperator*{\argmin}{argmin}

\let\mathsf\relax    
\DeclareRobustCommand{\mathsf}[1]{\text{\normalfont\sffamily#1}}

\newcommand{\msf}[1]{\mathsf{#1}}

\newcommand{\X}{{\mathcal X}}
\newcommand{\T}{{\mathcal T }}
\newcommand{\Y}{{\mathcal Y}}
\newcommand{\Z}{{\mathcal Z}}

\newcommand{\hh}{{\mathcal{H}}}

\newcommand{\push}{{\#}}

\newcommand{\F}{{\mathcal{F}}}

\newcommand{\pp}{{\mathcal{M}}}

\newcommand{\cont}{{\mathcal C}}

\newcommand{\prob}{\mathcal{P}}

\newcommand{\smooth}{\mathsf{\mathsf{R}}}
\newcommand{\bounded}{\mathsf{\mathsf{L}}}
\def\nor #1{\Vert #1 \Vert}

\newcommand{\subt}{\scalebox{0.4}{$T$}}
\newcommand{\subv}{\scalebox{0.4}{$V$}}

\renewcommand{\paragraph}[1]{{\bfseries #1.}}

\crefname{assumption}{Assumption}{Assumptions}
\crefname{equation}{}{}
\Crefname{equation}{Eq.}{Eqs.}
\crefname{figure}{Fig.}{Fig.}
\crefname{table}{Table}{Tables}
\crefname{section}{Sec.}{Sec.}
\crefname{theorem}{Thm.}{Thm.}
\crefname{proposition}{Prop.}{Prop.}
\crefname{fact}{Fact}{Facts}
\crefname{lemma}{Lemma}{Lemmas}
\crefname{corollary}{Cor.}{Cor.}
\crefname{example}{Example}{Examples}
\crefname{remark}{Remark}{Remarks}
\crefname{algorithm}{Alg.}{Algorithms}

\providecommand{\scal}[2]{\left\langle{#1},{#2}\right\rangle}

\providecommand{\nor}[1]{\left\|{#1}\right\|}

\providecommand{\abs}[1]{\left|{#1}\right|}

\providecommand{\norm}[1]{\lVert#1\rVert}
\providecommand{\scal}[2]{\left\langle{#1},{#2}\right\rangle}

\newcommand{\kl}{\textnormal{KL}}

\newcommand{\eps}{\varepsilon}
\newcommand{\sig}{\delta}

\renewcommand{\sink}{{\msf{S}_\eps}}
\newcommand{\fS}{\mathsf{S}}
\newcommand{\E}{\mathbb{E}}

\title{\LARGE\bf Generalization Properties of Optimal Transport GANs with Latent Distribution Learning\vspace{1em}}

\author{ Giulia Luise$^{1}$ \\ {\footnotesize\em g.luise.16@ucl.ac.uk} \and   Massimiliano Pontil$^{1,2}$ \\ {\footnotesize\em  m.pontil@cs.ucl.ac.uk ~~} \and  Carlo Ciliberto$^{1,3}$ \\ {\footnotesize\em c.ciliberto@imperial.ac.uk} \\ $ $ \\  }

\begin{document}

\maketitle

\begin{abstract}
\noindent The Generative Adversarial Networks (GAN) framework is a well-established paradigm for probability matching and realistic sample generation. While recent attention has been devoted to studying the theoretical properties of such models, a full theoretical understanding of the main building blocks is still missing. Focusing on generative models with Optimal Transport metrics as discriminators, in this work we study how the interplay between
the latent distribution and 
the complexity of the pushforward map (generator) affects performance, from both statistical and modeling perspectives. Motivated by our analysis, we advocate learning the latent distribution as well as the pushforward map within the GAN paradigm. We prove that this can lead to significant advantages in terms of  sample complexity. 
\end{abstract}

\section{Introduction}
\footnotetext[1]{Computer Science Department, University College London, WC1E 6BT London, United Kingdom}\footnotetext[2]{Computational Statistics and Machine Learning - Istituto Italiano di Tecnologia, 16100 Genova, Italy}\footnotetext[3]{Electrical and Electronics Engineering Department, Imperial College London, SW7 2BT, United Kingdom.}
Generative Adversarial Networks (GAN) are powerful methods for learning probability measures and performing realistic sampling \cite{goodfellow2014generative}. Algorithms in this class aim to reproduce the sampling behavior of the target distribution, rather than explicitly fitting a density function. This is done by modeling the target probability as the pushforward of a probability measure in a latent space. Since their introduction, GANs have achieved remarkable progress. From a practical perspective, a large number of model architectures have been explored, leading to impressive results in data generation \cite{vondrick2016generating,isola2017image,ledig2017photo}. From the side of theory, attention has been devoted to identify rich metrics for generator training, such as $f$-divergences \cite{nowozin2016f}, integral probability metrics (IPM) \cite{dziugaite2015training} or optimal transport distances \cite{arjovsky2017wasserstein}, as well as studying their approximation properties \cite{liu2017approximation,bai2018approximability,zhang2018on}. From the statistical perspective, progression has been slower. While recent work has set the first steps towards a characterization of the generalization properties of GANs with IPM loss functions \cite{arora2017generalization,zhang2018on,liang2018well, singh2018nonparametric,uppal2019nonparametric}, a full theoretical understanding of the main building blocks of the GAN framework is still missing.

 
In this work, we focus on optimal transport-based loss functions \cite{genevay2018learning} and  study the impact of two key quantities of the GANs paradigm on the overall generalization performance. In particular, we prove an upper bound on the learning rates of a GAN in terms of a notion of complexity for: $i)$ the ideal generator network and $ii)$ the latent space and distribution. Our result indicates that fixing the latent distribution a-priori  and offloading most modeling complexity on the generator network can have a disruptive effect on both statistics and computations. 

Motivated by our analysis, we propose a novel GAN estimator in which the generator and latent distributions are learned jointly. Our approach is in line with previous work on multi-modal GANs \cite{ben2018gaussian, pandeva2019mmgan}, which models the latent distribution as a Gaussian mixture whose parameters are inferred during training. 
In contrast to the methods above, our estimator is not limited to Gaussian mixtures but can asymptotically learn any sub-Gaussian latent distribution. Additionally, we characterize the learning rates of our joint estimator and discuss the theoretical advantages over performing standard GANs training, namely fixing the latent distribution a-priori.

\textbf{Contributions.} The main contributions of this work include: $i)$ showing how the regularity of a generator network (e.g. in terms of its smoothness) affects the sample complexity of pushforward measures and consequently of the GAN estimator; $ii)$ introducing a novel algorithm for joint training of generator and latent distributions in GAN; $iii)$ study the statistical properties (i.e. learning rates) of the resulting estimator in both settings where the GAN model is exact as well as the case where the target distribution is only approximately supported on a low dimensional domain.

The rest of the paper is organized as follows: \cref{sec:background} formally introduces the GANs framework in terms of pushforward measures.  \cref{sec:motivation} discusses the limitations of fixing the latent distribution a-priori and motivates the proposed approach. \cref{sec:theory} constitutes the core of the work, proposing the joint GAN estimator and studying its generalization properties. Training and sampling strategies are discussed in  \cref{sec:implementation}.  Finally \cref{sec:experiments} presents  preliminary experiments highlighting the effectiveness of the proposed estimator, 
while \cref{sec:conclusion} offers concluding remarks and potential future directions.

\section{Background}\label{sec:background}

The goal of probability matching is to find a good approximation of a distribution $\rho$ given only a finite number of points sampled from it. Typically, one is interested in finding a distribution $\hat\mu$ in a class $\pp$ of probability measures that best approximates $\rho$, ideally minimizing
\eqal{\label{eq:general-problem}
    \inf_{\mu\in\pp}~ \msf{d}(\mu,\rho). 
}
Here $\msf d$ is a discrepancy measure between probability distributions. A wide range of hypotheses spaces $\pp$ have been considered in the literature, such as space of distributions parametrized via mixture models \cite{dempster1977maximum,bishop2006pattern,sugiyama2010least,sriperumbudur2017density}, deep belief networks \cite{hinton2006fast,van2016conditional,van2016pixel} and variational autencoders \cite{kingma2013auto} among the most well known approaches. In this work we focus on generative adversarial networks (GAN)  \cite{goodfellow2014generative}, which can be formulated as in \cref{eq:general-problem} in terms of adversarial divergences and pushfoward measures. Below, we introduce these two notions and the notation used in this work. 

\paragraph{Adversarial divergences} Let $\P(\X)$ be the space of probability measures over a set $\X\subset\R^d$.
Given a space $\F$ of functions $F:\X\times\X\to\R$, we define the adversarial divergence between $\mu,\rho\in\P(\X)$
\eqal{\label{eq:adversarial-losses}
    \msf{d}_\F(\mu,\rho) ~=~ \sup_{F\in\F}~ \int F(x',x)~d(\mu\otimes\rho)(x',x),
}
namely the supremum over $\F$ of all expectations $\mathbb{E}~[F(x,x')]$ with respect to the joint distribution $\mu\otimes\rho$ (see \cite{liu2017approximation}). A well-established family of adversarial divergences are {\itshape integral probability metrics (IPM)}, where $F(x',x) = f(x')-f(x)$ for $f:\X\to\R$ in a suitable space (e.g. a ball in a Sobolev space). Here $\msf{d}_\F$ measures the largest  gap $\mathbb{E}_\mu f(x) - \mathbb{E}_\rho f(x)$ between the expectations of $\mu$ and $\rho$. Examples of IPM used in the GAN paradigm include Maximum Mean Discrepancy \cite{dziugaite2015training} and Sobolev-IPM \cite{mroueh2017sobolev}. Additional adversarial divergences are $f$-divergences \cite{nowozin2016f,goodfellow2014generative}. Recently, Optimal Transport-based adversarial divergences have attracted significant attention from the GANs literature, such as the Wasserstein distance \cite{arjovsky2017wasserstein}, the Sliced-Wasserestein distance \cite{liutkus2018sliced, nadjahi2019asymptotic, deshpande2018generative,wu2019sliced} or the Sinkhorn divergence \cite{genevay2018learning, sanjabi2018convergence}. For completeness, in \cref{app:adversarial-losses} (see also \cite{liu2017approximation}) we review how to formulate the adversarial divergences mentioned above within the form of \cref{eq:adversarial-losses}. 

\paragraph{Pushforward measures} Pushforward measures are a central component of the GAN paradigm. We recall here the formal definition while referring to \cite[Sec. 5.5]{ambrosio2008gradient} for more details. 

\begin{definition}[\textbf{Pushforward}]
Let $\Z$ and $\X$  be two measurable spaces and $T:\Z \to \X$ a measurable map. Let $\eta\in\mathcal{P}(\Z)$ be a probability measure over $\Z$. The \textbf{pushforward} of $\eta$ via $T$ is defined to be the measure $T_\push \eta$ in $\mathcal{P}(\X)$ such that for any Borel subset $B$ of $\X$,
\begin{equation}\label{eq:def1_pushy}
    (T_\push\eta)(B) ~=~ \eta(T^{-1}(B)).
\end{equation}
\end{definition}
To clarify the notation, in the rest of the paper we will refer to a measure $T_\push \eta$ as pushforward \textit{measure}, and to the corresponding $T$ as pushforward \textit{map}. A key property of pushforward measures is the {\itshape Transfer lemma} \cite[Sec 5.2]{ambrosio2008gradient}, which states that for any measurable $f:\X\to\R$,
\begin{equation}\label{eq:transfer_lemma}
    \int_{\X} f(x)\,d(T_{\push}\eta)(z) ~=~ \int_{\Z} f(T(z))\,d\eta(z).
\end{equation}
This property is particularly useful within GAN settings as we discuss in the following.

\paragraph{Generative adversarial networks} 
The generative adversarial network (GAN) paradigm consists in parametrizing the space $\pp$ of candidate models in \cref{eq:general-problem} as {\itshape a set of pushforwards measures of a latent distribution}. From now on, $\Z$ and $\X$ will denote  latent and target spaces and we will assume $\Z\subset \R^k$ and $\X\subset \R^d$. Given a set $\T$ of functions $T:\Z\to\X$ 
and given a (latent) probability distribution $\eta\in\P(\Z)$, we consider the space
\eqals{
    \pp(\T,\eta) ~=~ \{~ \mu = T_\push \eta, ~|~   T\in\T ~\}.
}
While this choice allows to parameterize the target distribution only implicitly, it offers a significant advantage at sampling time: sampling $x$ from $\mu = T_\push \eta$ corresponds to sample a $z$ from $\eta$ and then take $x = T(z)$. By leveraging the Transfer lemma \cref{eq:transfer_lemma} and using an adversarial divergence $\msf{d}_\F$, the probability matching problem in \cref{eq:general-problem} recovers the original minimax game formulation in \cite{goodfellow2014generative}
\eqals{\label{eq:gan-problem}
    \inf_{\mu\in\pp(\T,\eta)}~\msf{d}_\F(\mu,\rho) ~=~ \inf_{T\in\T}\msf{d}_\F(T_\push \eta,\rho) ~=~ \inf_{T\in\T}\sup_{F\in\F}~\int F(T(z),x)~d(\eta\otimes\rho)(z,x).
}
Within the GAN literature, the pushforward $T$ is referred to as the {\itshape generator} and optimization is performed over a suitable $\T$ (e.g. a set of neural networks \cite{goodfellow2014generative}) for a fixed $\eta$ (e.g. a Gaussian or uniform distribution). The term $F$ is called {\itshape discriminator} since, when for instance $\msf{d}_\F$ is an IPM, $F(x',x) = f(x') - f(x)$ aims at maximally separating (discriminating) the expectations of $\mu$ and $\rho$. 

\section{The Complexity of Modeling the Generator}\label{sec:motivation}

In this section we discuss a main limitation of choosing the latent distribution a-priori within the GANs framework. This will motivate our analysis in \cref{sec:theory} to learn the latent distribution jointly with the generator.  Let $\rho\in\P(\R^d)$ be the (unknown) target distribution and $\rho_n = \frac{1}{n}\sum_{i=1}^n \delta_{x_i}$ an empirical distribution of $n$ Dirac's deltas $\delta_{x_i}$ centered on i.i.d. points $(x_i)_{i=1}^n$ sampled from $\rho$. Given a latent distribution $\eta\in\P(\R^k)$ (such as Gaussian or uniform distribution), GAN training consists in learning a map $\hat T$ that approximately minimizes
\eqal{\label{eq:erm-gans}
    \hat T ~=~ \argmin_{T\in\T} ~\msf d_\F(T_\push \eta, \rho_n),
}
or a stochastic variant of such objective. The potential downside of this strategy is that it offloads all the complexity of modeling the target $\rho$ onto the generator $T$. Therefore, for a given $\eta$, it might be possible that the equality $T_\push\eta = \rho$ is satisfied only by very complicated -- hence hard to learn -- pushforward maps. To illustrate when this might happen and its effects on modeling and learning, we discuss some examples below (see \cref{app:pushforwards} for technical details). We first recall a characterization of pushforward measures that will be instrumental in building this intuition.
\begin{proposition}[Simplified version of {\cite[Lemma 5.5.3]{ambrosio2008gradient}}]\label{prop:push-ode}
Let $\rho$ and $\eta\in\mathcal{P}(\R^d)$ admit density functions $f,g:\R^d\to\R$ with respect to the Lebesgue measure, denoted $\eta = f\mathcal{L}^d$ and $\rho = g \mathcal{L}^d$. Let $T:\R^d \rightarrow \R^d$ be injective a.e. and differentiable, then $\rho = T_\push \eta$ if and only if \begin{equation}\label{eq:pushforward_ODE}
    g(T(x))\abs{\textnormal{det}\nabla T} ~=~ f(x).
\end{equation}
\end{proposition}
%
%
%

Under \cref{prop:push-ode} we can interpret the GAN problem as akin to {\itshape solving the differential equation \cref{eq:pushforward_ODE}}. Therefore, choosing $\eta$ a-priori might implicitly require a very complex model space $\T$ to find such a solution. In contrast, the following example describes the case where $\T$ contains only simple models.  

\begin{example}[Affine Pushforward Maps] Let $\eta\in\P(\R^d)$ admit a density $f:\R^d\to\R$ and let 
\eqal{\label{eq:class_of_linear_push}
    \mathcal{T} ~=~ \{~T_{A,b}: \R^d\to\R^d ~|~ T_{A,b}(x) = Ax + b,~ A\in\R^{d\times d}~ b\in\R^{d} \,\,\ \det(A)\neq0~\}.
}
Then, for $T = T_{A,b}\in\T$, the measure $\rho = T_{\push}\eta \in\P(\R^d)$ admits a density $g:\R^d\to\R$ such that
\eqal{\label{eq:description_push}
    g(x) ~=~ f(A^{-1}(x-b)) \cdot |\det(A^{-1})|.
}
\end{example}

The set $\T$ of affine generators is able to parametrize only a limited family of distributions (essentially translations and re-scaling of the latent $\eta$). This prevents significant changes to the shape of the latent distribution to match the target; for example, a uniform (alternatively a Gaussian) measures can match only uniform (Gaussian) measures.  Below, we illustrate two examples where the pushforward map can indeed be quite complex and therefore require a larger space $\T$ when solving \cref{eq:erm-gans}. 

\begin{example}[Uniform to Gaussian]
Let $\eta$ be the uniform distribution on the interval $[-1,1]$ and $\rho$ the Gaussian distribution on $\R$, with zero mean and unit variance. Then $T_\push \eta = \rho$, with $T(x) = \sqrt{2}\textnormal{erf}^{-1}(x)$ the inverse to the standard error function $\textnormal{erf}(x) = \rho\big((-\infty,x]\big)$. 
\end{example}
The map $\textnormal{erf}^{-1}$ is highly nonlinear and with steep derivatives. 
Therefore, learning a GAN from a uniform to a Gaussian distribution would require choosing a significantly large space $\T$ to approximate $\textnormal{erf}^{-1}$. We further illustrate the effect of a similar mismatch with an additional empirical example.

{{\bfseries Empirical example} (Multi-modal Target){\bfseries .}}
 We consider the case where $\rho$ is multimodal (a mixture of four Gaussian distributions in $2$D), while $\eta$ is unimodal (a Gaussian in $2$D). \cref{fig:example-complexity-push} qualitatively compares samples from the real distribution $\rho$ against samples from $T_\push\eta$, with $T$ learned via GAN training in \cref{eq:erm-gans} (with Sinkhorn loss $\msf d_\F$, see \cref{sec:theory}) for spaces $\T$ of increasing complexity (neural networks with increasing depth). Linear generators are clearly unsuited for this task, and only highly non-linear models yield reasonable estimates. See \cref{app:pushforwards} for details on the experimental setup.

\begin{figure*}[t!]
    \centering
    \hspace{-1.8em}
    \begin{subfigure}[t]{0.23\textwidth}
        \centering
        \includegraphics[height=0.7\textwidth]{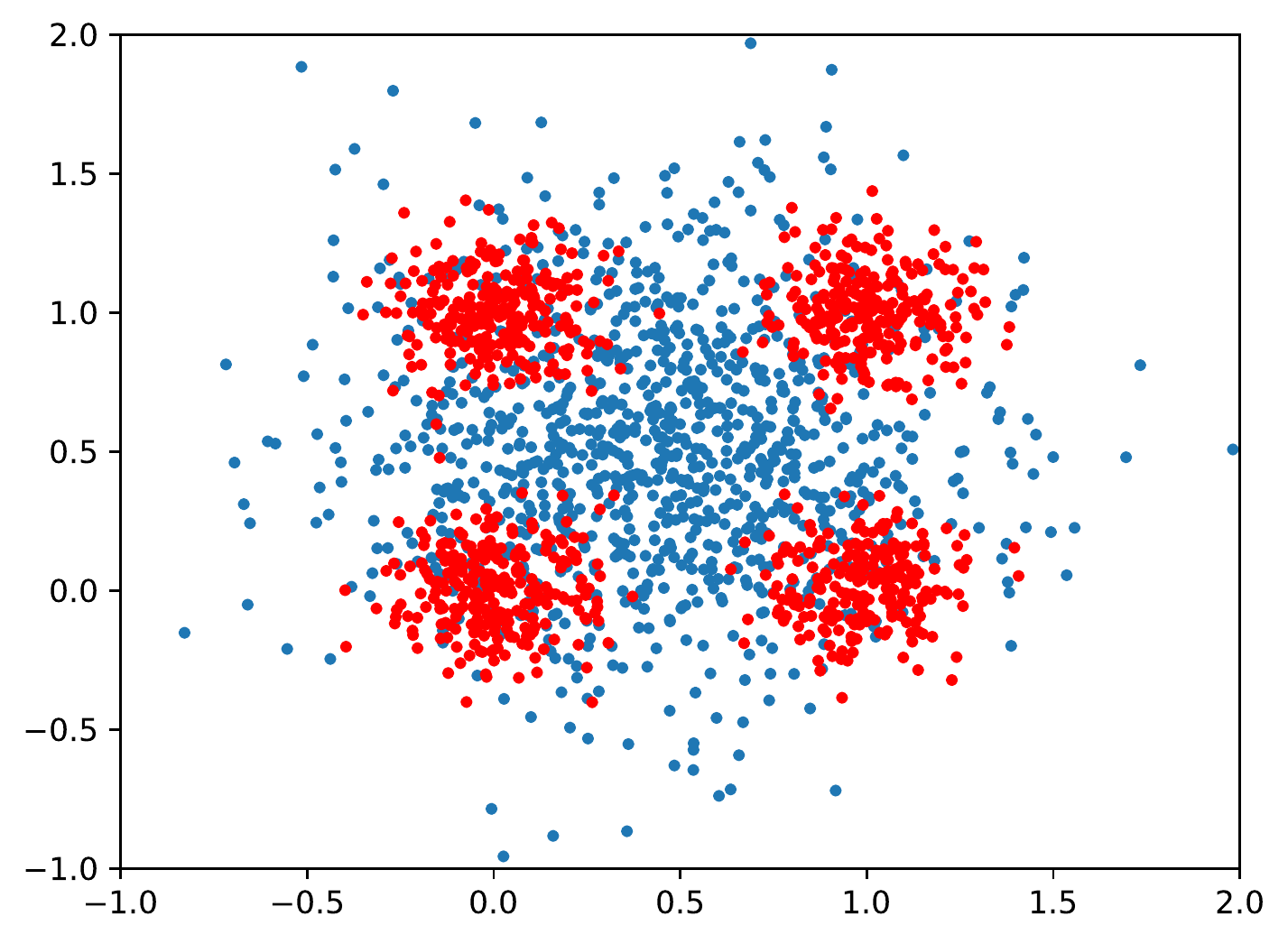}
        \caption{1 layer}
    \end{subfigure}%
    ~
    \begin{subfigure}[t]{0.23\textwidth}
        \centering
        \includegraphics[height=0.7\textwidth]{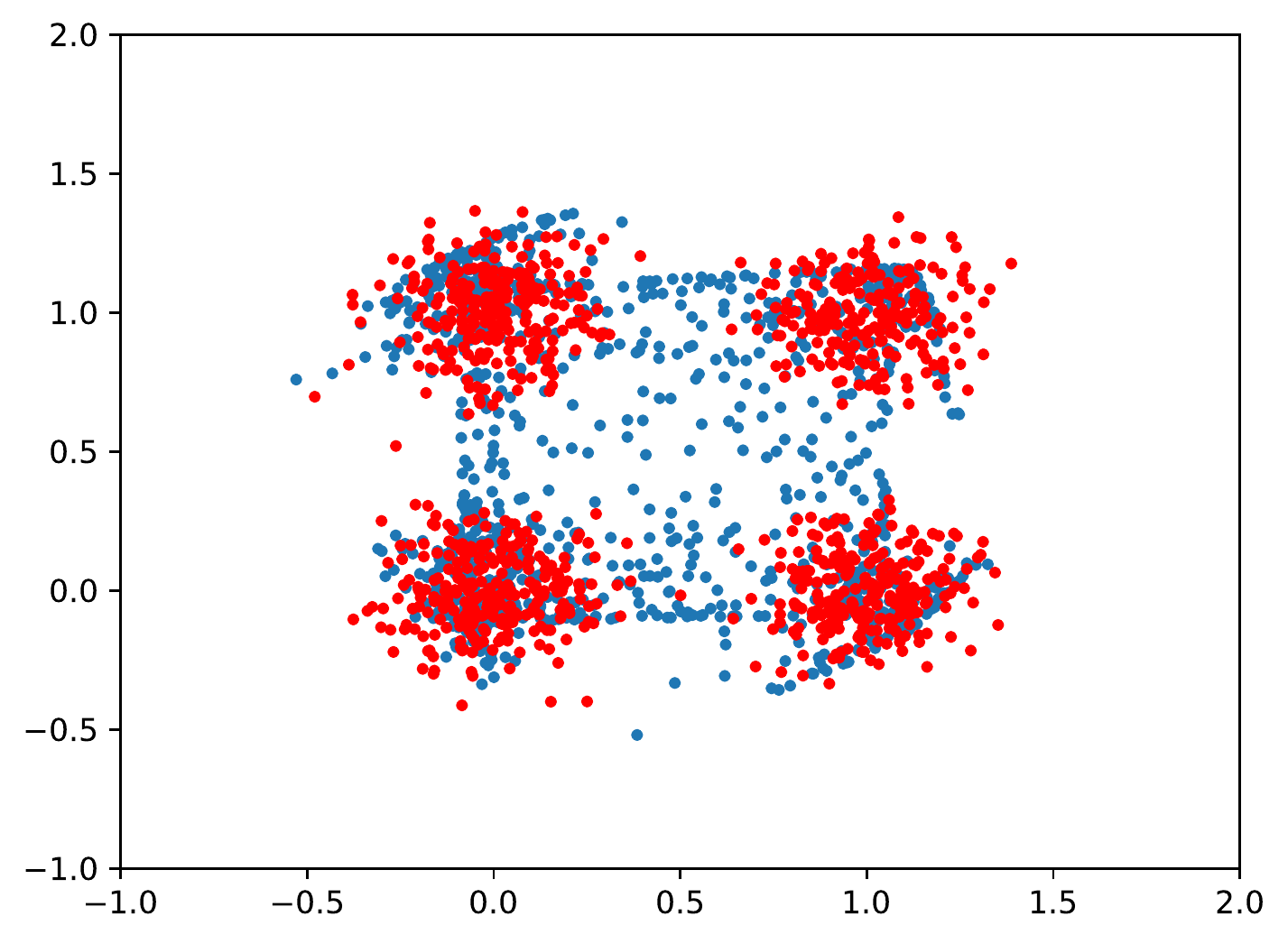}
        \caption{2 layers}
    \end{subfigure}%
    ~
    \begin{subfigure}[t]{0.23\textwidth}
        \centering
        \includegraphics[height=0.7\textwidth]{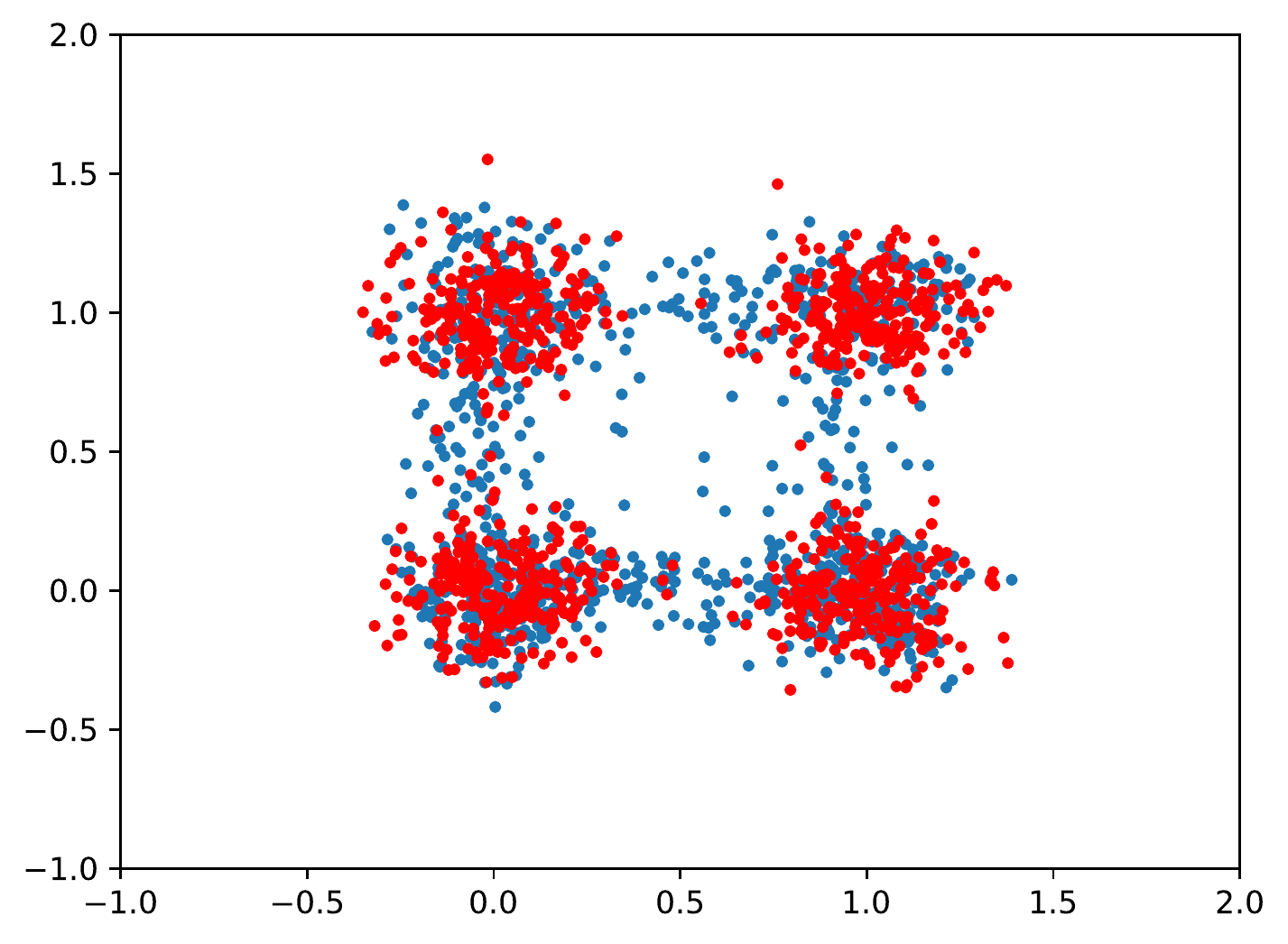}
        \caption{5 layers}
    \end{subfigure}%
    ~
    \begin{subfigure}[t]{0.23\textwidth}
        \centering
        \includegraphics[height=0.7\textwidth]{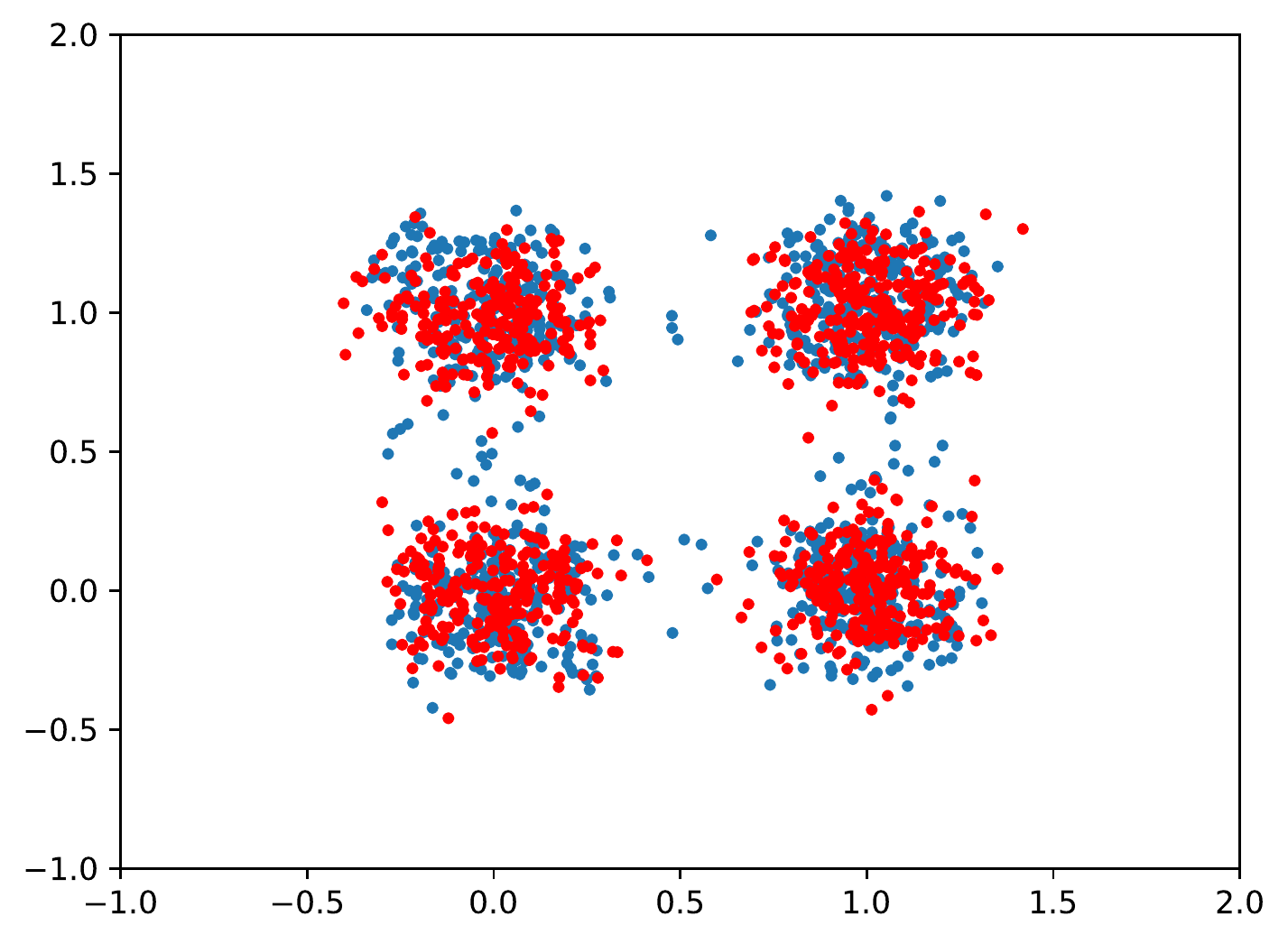}
        \caption{7 layers}
    \end{subfigure}
    \caption{Sinkhorn GAN estimation between a $2$D Gaussian and a mixture of four $2$D Gaussians with generator space $\T$ of increasing complexity (depth). Real (Red) vs generated (Blue) samples. \label{fig:example-complexity-push}}
\end{figure*}

\section{Learning the Latent Distribution}\label{sec:theory}

The arguments above suggest that choosing the latent distribution a-priori can be limiting in several settings. Therefore, in this work we propose to learn the latent distribution jointly with the generator. Given a family $\hh$ of latent distributions, we aim to solve
\eqal{\label{eq:erm-joint-gan}
    (\hat T,\hat \eta) ~=~ \argmin_{T\in\T,\eta\in\hh} ~ \msf{d}_\F(T_\push\eta,\rho_n).
}
A natural question is how the learning rates of $(\hat T,\hat \eta)$ are affected by the choice of $\T$ and $\hh$. In this work we address this question for the case where $\msf{d}_\F$ is the Sinkhorn divergence \cite{cuturi2013sinkhorn}. Indeed, Optimal Transport is particularly suited to capture the geometric properties of distribution supported on low-dimensional manifolds \cite{weed2019sharp} (e.g. pushforward measures from a low-dimensional latent space). Moreover, for the Sinkhorn diveregence, discriminator training, i.e. finding $F$ in \cref{eq:gan-problem}, can be efficiently solved to arbitrary precision via the Sinkhorn-Knopp algorithm (see \cite{cuturi2013sinkhorn} and \cref{sec:implementation}). Below, we introduce our choices for $d_\F, \hh$ and $\T$ and then proceed to characterize the learning rates of $\hat T_\push \hat\eta$.

\paragraph{Choosing $\msf{d}_\F$: Sinkhorn divergence}
For any $\mu,\nu \in \prob(\dom)$, the Optimal Transport problem with entropic regularization is defined as follows \cite{peyre2017computational,cuturi2013sinkhorn,genevay2016}
\begin{equation}\label{eq:primal_pb}
\oteps(\mu,\nu) ~=~ \min_{\pi\in\Pi(\mu,\nu)}~\int \nor{x-y}^2\,d\pi(x,y) + \eps\kl(\pi|\mu\otimes\nu),\qquad \eps\geq0
\end{equation}
where $\kl(\pi|\mu\otimes\nu)$ is the \emph{Kullback-Leibler divergence} between the candidate transport plan $\pi$ and the distribution  $\mu \otimes \nu$, and $\Pi(\mu,\nu)=\{\pi\in\prob(\dom^2)\colon \mathsf{P}_{1\push}\pi=\mu,\,\,\mathsf{P}_{2\push}\pi=\nu\}$, with $\mathsf{P}_{i}\colon\dom\times \dom\rightarrow\dom$ the projector onto the $i$-th component and $\push$ the push-forward operator. For $\eps = 0$, $\oteps$ corresponds to the $2$-Wasserstein distance \cite{peyre2017computational}. The Sinkhorn {\em divergence} $\sink\colon\prob(\dom)\times\prob(\dom)\to\R$ is defined as
\begin{equation}\label{eq:sink_divergence}
\sink(\mu,\nu) ~=~ \oteps(\mu,\nu) ~-~ \frac{1}{2}\oteps(\mu,\mu) ~-~\frac{1}{2}\oteps(\nu,\nu),
\end{equation}
and can be shown to be nonnegative, biconvex and to metrize the convergence in law \cite{feydy2018interpolating}. Entropic regularization was originally introduced as a computationally efficient surrogate to evaluating the Wasserstein distance \cite{cuturi2013sinkhorn}. Recent work showed that the Sinkhorn divergence is also advantageous in terms of sample complexity, with better dependency to the ambient space dimension \cite{genevay2018sample,weed2019sharp}. 

\paragraph{Choosing $\hh$: sub-Gaussian distributions}
For the purpose of our analysis, in the following we will restrict to a class $\hh$ of distribution that are not too spread out on the entire latent domain $\Z\subseteq \R^k$. In particular, we will parametrize $\hh \subset \mathcal{G}_\sigma(\Z)$ the space of $\sigma$-{\itshape sub-Gaussian} distributions on $\Z$, namely distributions $\eta$, such that  $\int e^{\nor{z}^2/2k\sigma^2}~d\eta(z) \leq 2$. Gaussian distributions and probabilities supported on a compact set belong to this family.
%
%
Thus, $\hh$ recovers the case of standard GANs. Note  that the parameter $\sigma$ allows us to upper bound all moments of a distribution and can therefore be interpreted as a quantity that controls the complexity of $\eta$.

\paragraph{Choosing $\T$: balls in $C^s(\Z,\X)$}
In the following we will restrict our analysis to spaces of functions that satisfy specific regularity conditions. In particular, we will consider $\T \subset C_{\tau,L}^s(\Z,\X)$ to be contained in the set of $L$-Lipscthitz functions in the ball of radius $\tau$ in the space of continuous functions from $\Z\subset\R^k$ to $\X\subset\R^d$ equipped with the uniform norm $\nor{\cdot}_{\infty,s}$ on all partial derivatives up to order $s$. Intuitively, the norm $\nor{T}_{\infty,s}$ quantifies the complexity of the generator $T$, hence reflecting how easy (or hard) it is to learn it in practice. In our analysis we will require $s\geq\lceil k/2\rceil + 1$.
To simplify our analysis in the following, we add the additional requirements that $T(0) = 0$ for any $T\in\T$ (in case of an offset, one can factor out a translation first (see \cite[Remark 2.19]{peyre2017computational}).   This choice of the space $\T$ allows to formally model a wide range of smooth generators $T$ (e.g. pushforward maps parametrized by neural networks with smooth activations or via smooth reproducing kernels).


%
We are now ready to state our main result, which characterizes the learning rates of the estimator in \cref{eq:erm-joint-gan} in terms of the complexity parameters associated to the spaces $\T$ and $\hh$ introduced above.

\begin{restatable}{theorem}{thmMainRates}\label{thm:main-rates}
Let $\Z\subset\R^k$, $\X\subset\R^d$ and $\rho = T^*_\push \eta^*$ with  $T^*\in\T \subset C_{\tau,L}^{\lceil k/2\rceil + 1}(\Z,\X)$ and $\eta^*\in\hh \subset \mathcal{G}_\sigma(\Z)$. Let $(\hat T,\hat \eta)$ satisfy \cref{eq:erm-joint-gan} with $\msf{d}_\F = \sink$ and $\rho_n$ a sample of $n$ i.i.d. points from $\rho$. Then,
\eqals{
    \mathbb{E}~\sink(\hat T_\push\hat\eta,\rho) ~\leq~ \frac{\msf c(\tau,L,\sigma,k)}{\sqrt{n}}
}
where $\msf c(\tau,L,\sigma,k) ~=~  \msf C_{k}~(1+ \tau^k L^{\lceil 3k/2\rceil + 1}~\sigma ^{\lceil 5k/2 \rceil + 6} ~\eps^{-\lceil 5k/4  \rceil - 3 })$ with $\msf C_k$ a constant depending only on the latent space dimension $k$ and where the expectation is taken with respect to $\rho_n$.
\end{restatable}

\cref{thm:main-rates} quantifies the tradeoff between the complexity terms $\tau, L$ of the pushforward map $T^*$ and $\sigma$ of the latent distribution $\eta^*$. In particular, we note that: $i)$ we pay a polynomial cost in terms of the sub-Gaussian parameter $\sigma$ of the latent distribution $\eta^*$; $ii)$ we pay a cost proportional to the complexity of the target generator, including its $\nor{T^*}_{s,\infty}$ norm and Lipschitz constant $L$. $iii)$ all terms depend on the dimension $k$ of the latent space and {\itshape not} on the target space dimension $d$. This result suggests that the GAN paradigm is particularly suited to settings where the target distribution can be modeled in terms of a low-dimensional latent distribution {\itshape and} a regular pushforward map. Extending the result to larger families of pushforward maps (e.g. with weaker regularity assumptions) will be the subject of future work. 


\begin{proof}[Sketch of the proof]
We report the proof of \cref{thm:main-rates} in \cref{app:rates}. Here we present the main steps and key ideas. We begin by observing that leveraging the optimality of the estimator $\hat T_\push \hat \eta$ in minimizing \cref{eq:erm-joint-gan}, the matching error is controlled by
\eqal{\label{eq:upper-bound-sample-complexity}
    \sink(\hat T_\push\hat\eta,\rho) ~\leq~ 2 \sup_{T\in\T,\eta\in\hh}~ \abs{\sink(T_\push\eta,\rho)-\sink(T_\push\eta,\rho_n)}.
}
The right hand side corresponds to the largest generalization error of estimators in $(\T,\hh)$. This quantity is related to the sample complexity of $\rho_n$ with respect to the Sinkhorn divergence. The latter is a topic recently studied in \cite{genevay2018sample,mena2019statistical}, with bounds available for controlling $\abs{\sink(\mu,\rho)-\sink(\mu,\rho_n)}$ for $\mu$ a fixed distribution. However, to control \cref{eq:upper-bound-sample-complexity} we need to provide a uniform upper bound for the sample complexity of the Sinkhorn divergence over the class $(\T,\hh)$. To do so, we use the following.
%
%
\begin{lemma}[Informal]\label{lem:informal-uniform-upper-bound-space-of-functions}
Let $\eta,\nu_1,\nu_2\in\mathcal{G}_\sigma(\Z)$ and $T,T'\in \T$ with $\T$ as in \cref{thm:main-rates}. 
Then,
\eqal{\label{eq:upper-bound-informal-lemma}
    \abs{\sink(T_\push\eta,T'_\push\nu_1)-\sink(T_\push\eta, T'_\push\nu_2)}~\leq~ \sup_{f\in\F_{\sigma,\tau,L}}\abs{\int f(z)~d\nu_1 - \int f(z)~d\nu_2(z)}
}
with $\F_{\sigma,\tau,L}$ a suitable space of functions $f:\Z\to\R$ that does not depend on $\eta, T$ and $T'$ but only on the complexity parameters $\sigma,\tau$ and $L$ (see  \cref{app:technical-results} for the characterization of $\F_{\sigma,\tau,L}$).

\end{lemma}\vspace{-5pt}

The result implies that we can upper bound the generalization error of $T_\push\eta$ in terms of the integral probability metric $\msf{d}_{\F_{\sigma,\tau,L}}(\eta^*,\eta^*_n)$ (i.e. (right hand side of  \cref{eq:upper-bound-informal-lemma}) between the true latent $\eta^*$ and its empirical sample $\eta^*_n$. Note that this quantity is uniform with respect to the sub-Gaussian parameter of distributions in $\hh$ and the regularity of the class $\T$.
Following \cite[Thm. 2]{mena2019statistical}, we can control $\msf{d}_{\F_{\sigma,\tau,L}}(\eta^*,\eta^*_n)$ in expectation by estimating the covering numbers of a rescaling of $\F_{\sigma,\tau,L}$.
%
\end{proof}
\cref{thm:main-rates} studies the learning rates of the estimator in \cref{eq:erm-joint-gan} when the GAN model is exact. A natural question is whether similar results hold when this is only an approximation. Below, we consider the case where the target distribution is ``almost'' a low-dimensional pushforward (e.g. it is concentrated around a low-dimensional manifold) but is supported on a larger domain (e.g. due to noise).

\paragraph{Approximation error for noisy models} 
Let the target distribution $\rho$ be obtained by convolving $T^*_\push \eta^*$ with a distribution $\Phi_\sig$ with sub-Gaussian parameter $\delta>0$. Recall that the convolution $\Phi_\sig\ast \mu$ is defined as the distribution such that, for any measurable $f:\X\to\R$
\eqals{
    \int f(x)~d(\Phi_\sig\ast  \mu)(x) ~=~ \int f(w + y)~d\mu(y)d\Phi_\sig(w).
}
Therefore, $\rho = \Phi_\sig\ast T_\push\eta$ can be interpreted as the process of ``perturbing'' the distribution $T_\push\eta$ by means of a probability $\Phi_\sig$. Standard examples are the cases where Gaussian or uniform noise is added to samples from the pushforward. This perturbation affects the generalization of the proposed estimators by a term proportional to the sub-Gaussian parameter $\delta$ of the noise $\Phi_\sig$, as follows.
\begin{restatable}{corollary}{corPerturbation}\label{cor:perturbation}
Under the same assumption of \cref{thm:main-rates}, let $\Phi_\sig\in\mathcal{G}_\delta(\X)$ and $\rho = \Phi_\sig\ast T^*_\push\eta^*$. Let $\msf c$ be the same constant of \cref{thm:main-rates}. Then,
\eqals{
    \mathbb{E}~\sink(\hat T_\push \hat \eta, \rho) ~\leq~ \frac{2\,\msf c(\tau,L,\sigma,k)}{\sqrt{n}} ~+~ 3\tau d~\delta.
}
\end{restatable}

\cref{cor:perturbation} characterizes the approximation behavior of the GAN paradigm (see \cref{app:perturbation} for a proof). It shows that when the target distribution is essentially low-dimensional, we can recover it up to a quantity that depends on the intensity of the noise $\Phi_\sig$. This is reminiscent of the irreducible error in supervised learning settings when approximating a function lying outside the hypothesis space \cite{shalev2014understanding}. 

\section{Optimization}\label{sec:implementation}
In this section we discuss how to parametrize the spaces $\T$ and $\hh$ and tackle the joint GAN problem in practice. We note that minimizing $\sink(T_\push\eta,\rho_n)$ with respect to either $T$ or $\eta$ critically hinges on the dual formulation of entropic Optimal Transport. Therefore, we first review its formulation and main properties.  Then, we use such notion to address \cref{eq:erm-joint-gan}.

\paragraph{Dual Formulation of Entropic OT}
The dual formulation of \cref{eq:primal_pb} for $\oteps(\mu,\rho)$ is (see e.g. \cite{chizat2018scaling})
\eqals{\label{eq:dual_pb}
\sup_{u,v\in \cont(\X)}~\int u(x)~d\mu(x) ~+~ \int v(y)~d\rho(y) ~-~\eps \int e^{\frac{u(x) + v(y) - \|x-y\|^2}{\eps}}~d\mu(x)d\rho(y).
}
This problem always admits a pair of minimizers $(u^*,v^*)$, also known as {\itshape Sinkhorn potentials} \cite{sinkhorn1964}. When $\mu$ and $\nu$ are probability distributions with finite support, the well-established {\sc SinkhornKnopp} algorithm can be applied to efficiently obtain the values of $u^*$ and $v^*$ on the support points of $\mu$ and $\nu$ respectively \cite{sinkhorn1964,cuturi2013sinkhorn}. Then, the value of $u^*$ and $v^*$ can be evaluated on any point of $\X$ by means of the following characterization of the Sinkhorn potentials \cite{feydy2018interpolating}
\eqals{\label{eq:characterization-sinkhorn-potential}
    u^*(x) ~=~ -~\eps~\log \int e^{\frac{v^*(y) - \|x-y\|^2}{\eps} }~d\rho(y).
}
This characterization will be of particular interest in the following. Indeed both optimization of the generator and latent distribution will make use of the explicit calculation of the gradients of $u^*$. 

\paragraph{Learning the Generator} Minimizing \cref{eq:erm-joint-gan} with respect to $T$ for a fixed latent distribution $\eta$ corresponds to training a standard GAN. The case of Sinkhorn GANs was originally studied in \cite{genevay2018learning}. In practice, one considers a parametric family of generators $T_\theta$ with $\theta\in\Theta$. The gradients of Sinkhorn divergence $\sink({T_\theta}_\push\eta,\rho_n)$ with respect to  $\theta$ can be obtained via automatic differentiation. Here, we provide an analytic formula and discuss its potential benefits in \cref{app:gradient}. We will assume the parametrization to be differentiable a.e., and denote by $\nabla_\theta T_\theta$ the gradient of $T_\theta$ with respect to $\theta$. By leveraging the characterization of dual potential in \cref{eq:characterization-sinkhorn-potential}, we have the following.
%
\begin{restatable}{proposition}{PGeneratorGradient}\label{prop:generator-gradient}
Let $\eta\in\prob(\Z)$ and $\rho \in \prob(\X)$. Let $(u^*,v^*)$ be a pair of minimizers of \cref{eq:dual_pb} with $\mu = {T_\theta}_\push \eta$ and $\nu = \rho$. Then, the gradient of $\oteps({T_{\theta}}_\push\eta,\rho)$ in $\theta_0$ is
\eqals{\label{eq:generator-gradient-step}
\big[\nabla_{\theta}\oteps({T_{\theta}}_\push\eta,\rho)\big]|_{\theta ~=~ \theta_0} = \int \big[\nabla_x u^*(\cdot)\big]|_{x=T_{\theta_0}(z)}~ \big[\nabla_\theta T_\theta(z)\big]|_{\theta=\theta_0} ~d\eta(z).
}
\end{restatable}

The formula above is efficient when $\eta$ is discrete. When $\eta$ has dense support, computing the gradient of $\oteps$ with \cref{eq:generator-gradient-step} can become prohibitive. A potential approach proposed in \cite{genevay2018learning} is to sample $m$ points from $\eta$ and $\rho_n$ and compute $\nabla_\theta\oteps({T_\theta}_\push\eta_m,\rho_m)$ as a proxy of the target gradient.  


\paragraph{Learning the Latent Distribution} To guarantee $\hat\eta$ to be sub-Gaussian, we consider $\hh=\P(\Z)$ the set of probability measures over a compact subset $\Z$ of the latent space $\R^k$. Optimization over a space of measures $\P(\Z)$ is itself an active research topic. Possible strategies include Conditional Gradient \cite{bredies2013inverse,boyd2017alternating,mensch2019geometric,luise2019sinkhorn}, Mirror Descent \cite{hsieh2018finding} or the particle based approaches discussed below \cite{feydy2018interpolating,chizat2019sparse}.

\textit{Flow-based methods} approximate the target distribution with a set of $m$ particles $\eta = \sum_{i=1}^m \omega_i \delta_{z_i}$ whose position is then optimized to minimize $\sink(T_\push\eta,\rho_n)$ (for simplicity, here we do not learn the $\omega_i$ but fix them to $1/m$). This problem can be solved by a gradient descent-based algorithm in the direction minimizing the associated Sinkhorn potentials \cite{feydy2018interpolating}. More precisely, given $(u^*,v^*)$ a minimizer of \cref{eq:dual_pb}, we update the position of each particle via a gradient step of size $\alpha>0$
\eqal{\label{eq:flow-gradient-step}
z_+ ~=~ z ~-~ \alpha \nabla_z u^*(T(z)).
}
We refer to \cite{chizat2019sparse} for more details and a comprehensive analysis of convergence and approximation guarantees for particle-based methods with respect to $m$ the number of particles.


\begin{algorithm}[t]
\footnotesize
\caption{{\sc Latent Distribution Learning GANs}}
\label{alg:latent-GAN}
\begin{algorithmic}
\vspace{0.25em}
  \STATE {\bfseries Input:} Target $\rho_n$, latent dimension $k$,  initial network params $\theta$, Sinkhorn param $\eps>0$, step sizes $\alpha_1,\alpha_2>0$
  \STATE \phantom{\bfseries Input:} perturbation $\Phi_\sig$ (e.g. Gaussian $\mathcal{N}(0,\delta I_k)$ in $\R^k$), starting particles $(z_i)_{i=1}^m$, sampling size $\ell$.
  \vspace{0.45em}
  \STATE {\bfseries Until convergence do:} 
  \vspace{0.3em}
  \STATE \qquad Sample ${(i_j,w_j)_{j=1}^\ell}$ \qquad~~~ with \quad  $i_j\sim\textrm{Unif.}\{1,\dots,m\}$ \quad and \quad $w_j\sim \Phi_\sig$
  \vspace{0.1em}
  \STATE \qquad Let ${\mu ~=~ \frac{1}{\ell}\sum_{j=1}^\ell \delta_{x_j}}$ \qquad with \quad $x_j ~=~ T_\theta(z_{i_j} + w_j)$
  \vspace{0.1em}
  \STATE \qquad $(u^*,v^*) ~=~ ${\sc SinkhornKnopp}$(\mu, \rho_n,\eps)$
  \vspace{0.1em}
  \STATE \qquad ${\theta \,\,\gets \theta - \alpha_1 \sum_{j=1}^\ell \nabla_x u^*(x_j)~ \nabla_\theta T_\theta(z_{i_j}+w_j)}$
  \vspace{0.1em}
  \STATE \qquad ${z_i \gets z_i - \alpha_2 \sum_{j~|~i_j = i}\nabla_x u^*(x_j) ~ \nabla_z T_\theta(z_i + w_j)}$
  \vspace{0.45em}
\STATE {\bfseries Return:} \quad $\hat\mu = {T_\theta}_\push\hat\eta$ \quad with ${\hat\eta = \Phi_\sig\ast \hat\nu}$ \quad and ${\hat\nu = \frac{1}{m} \sum_{i=1}^m z_i}$.
  \vspace{0.1em} \STATE {\bfseries Sampling:} ~~ $x\sim\hat\mu$ \quad~~~ obtained as $x = T(z_i + w)$ \quad with $i\sim\textrm{Unif.}\{1,\dots,m\}$ \quad and $w\sim \Phi_\sig$. 
\end{algorithmic}
\end{algorithm}

\paragraph{Sampling $\&$ Training} 
Both conditional gradient and flow-based methods approximate the ideal $\hat\eta$ via a discrete distribution. While these strategies are guaranteed to approximate to arbitrary precision the ideal solution, they cannot be directly used for sampling new points. To this end here we propose to model $\eta = \Phi_\sig \ast \nu$ as the convolution of a discrete  $\nu = \frac{1}{m}\sum_{i=1}^m\delta_{z_i}$ with a $\delta$-variance Gaussian distribution $\Phi_\sig$. We can then address the following variant to the joint problem \cref{eq:erm-joint-gan}
\eqal{\label{eq:erm-joint-perturbed}
\min_{\theta\in\Theta, \nu \in \hh}~ \sink({T_\theta}_\push (\Phi_\sig \!\ast\! \nu), \rho_n),
}
where $\nu$ is learned by means of the flow-based approaches introduced above (by sampling a new set of points from $\Phi_\sig$ at each iteration). This strategy effectively renders the estimated $\eta$ to be a mixture of $m$ Gaussian distributions, whose position on the latent space is optimized iteratively.

\cref{alg:latent-GAN} summarizes the process of jointly learning $\eta = \Phi_\sig \ast \nu$ and $T_\theta$, according to  \cref{eq:erm-joint-perturbed}. For simplicity, we consider the case where we optimize simultaneously both network parameters and support points of the latent distribution. However other options are viable, such as block coordinate descent or alternating minimization, where each term is optimized while keeping the other fixed to the previous step. The algorithm proceeds iteratively by: $i)$ sampling points from the current estimate of $\hat\eta$ in terms of the discrete $\hat\nu$ and the perturbation $\Phi_\sig$; $ii)$ computing the Sinkhorn potential $u^*$ via the {\sc SinkhornKnopp}\footnote{We used the implementation from \cite{feydy2018interpolating} available at \url{https://www.kernel-operations.io/geomloss/}.} algorithm; $iii)$ update the network parameters $\theta$ and latent $\hat\nu$ according to the gradient steps \cref{eq:generator-gradient-step} and \cref{eq:flow-gradient-step} respectively. 

\section{Experiments}\label{sec:experiments}
\vspace{-0.1em}

We tested the proposed strategy of jointly learning the latent distribution and generator on two synthetic experiments. We do so by comparing the performance of the joint GAN estimator from \cref{alg:latent-GAN} and of the standard GAN estimator, with fixed latent distribution. We report both the qualitative sampling behavior of the two methods as well as their quantitative performance in terms of {\itshape generalization gap}, namely the value $\sink(T_\push\eta, \rho)$ attained at convergence (using the Sinkhorn distance between generated and {\itshape new} real samples as a proxy). Details on the setup, data generation, networks specifications and training are reported in \cref{app:sec_experiments}. \\

{\textbf{Spiral}.} We chose the target $\rho$ to be a multimodal probability measure in $\R^2$ supported on a spiral-shaped $1$D manifold (\cref{fig:spiral-target}, color intensity proportional to higher density). Given the low-dimensionality of the target, we consider a GAN model with latent space $\Z$ in $\R$. We compare \cref{alg:latent-GAN} against a GAN trained with fixed latent distribution $\eta_0 = \mathcal{N}(0,1)$ (the univariate Gaussian measure on $\R$) by training them on a sample $\rho_n$ of $n=1000$ i.i.d. points sampled from $\rho$. \cref{fig:spiral} reports the density and a sample of the ground-truth target $\rho$ (\cref{fig:spiral-target}), our $\hat T_\push\hat\eta$ estimated via \cref{alg:latent-GAN} (\cref{fig:spiral-ours}) and $\hat T'_\push\eta_0$ trained with standard Sinkhorn GAN (\cref{fig:spiral-fixed}). We see that the generator $\hat T'$ is unable to apply enough distortion to $\eta_0$ to match $\rho$ (see our discussion in \cref{sec:motivation}). In contrast, our method recovers the target distribution with high accuracy. This is quantitatively reflected by the generalization gaps: $\sink(\hat T_\push \hat \eta, \rho) <10^{-5}$ for \cref{alg:latent-GAN} and $\sink(\hat T'_\push \eta_0,\rho)>0.1$ for the fixed latent. \cref{fig:spiral-ours-latent} reports the density of the latent $\hat\eta$ on $\R$ to show how our method captures the bi-modality of $\rho$. \\


{\textbf{Swiss Roll}.}
Similarly to the previous setting, we consider $\rho$ a multimodal distribution in $\R^3$ supported on the $2$D swiss roll manifold. Latent space was set as $\Z\subset\R^2$. \cref{fig:swissroll} (Left to right) shows samples from the ground truth, our joint $\hat T_\push\hat\eta$ and the standard GAN $\hat T_\push\eta_0$ with $\eta_0 = \mathcal{N}(0,I)$. Also in this case, the latter generator was not able to fully recover the geometry of $\rho$, as indicated by the different generalization gaps $\sink(\hat T_\push \hat \eta, \rho)<0.5$ and $\sink(\hat  T'_\push \eta_0,\rho)>1.8$.\\


    

\begin{figure*}[t!]
    \centering
    \hspace{-2em}
    \begin{subfigure}[t]{0.3\textwidth}
        \centering
        \includegraphics[scale=0.17]{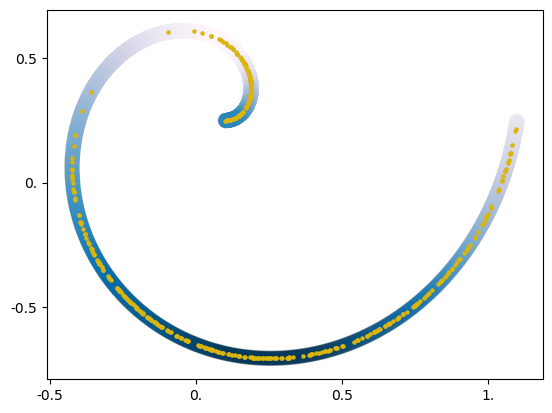}
        \caption{Target density\label{fig:spiral-target}}
    \end{subfigure}%
    \hspace{-1.45cm}
    \begin{subfigure}[t]{0.3\textwidth}
        \centering
        \includegraphics[scale=0.17]{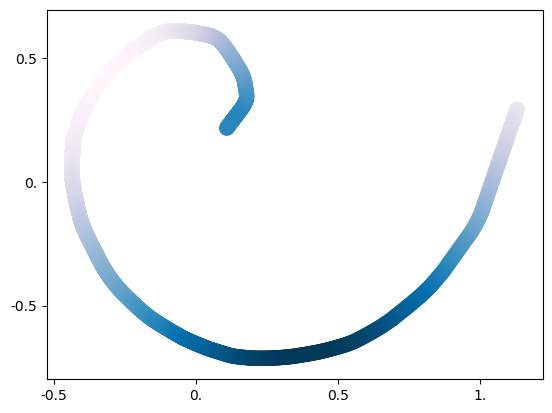}
        \caption{\cref{alg:latent-GAN}\label{fig:spiral-ours}}
    \end{subfigure}%
     \hspace{-1.45cm}
    \begin{subfigure}[t]{0.3\textwidth}
        \centering
        \includegraphics[scale=0.17]{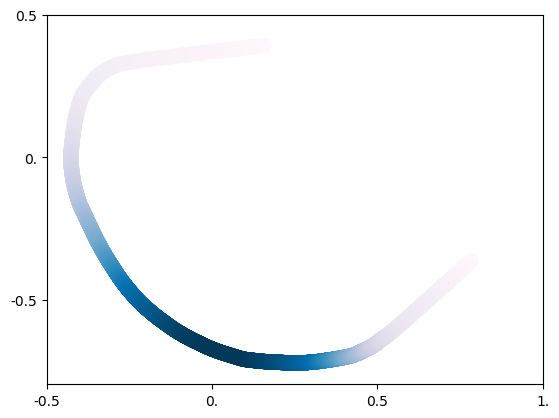}
        \caption{Fixed latent\label{fig:spiral-fixed}}
    \end{subfigure}%
    \hspace{-0.5 cm}
        \begin{subfigure}[t]{0.3\textwidth}
        \centering
        \includegraphics[scale=0.17]{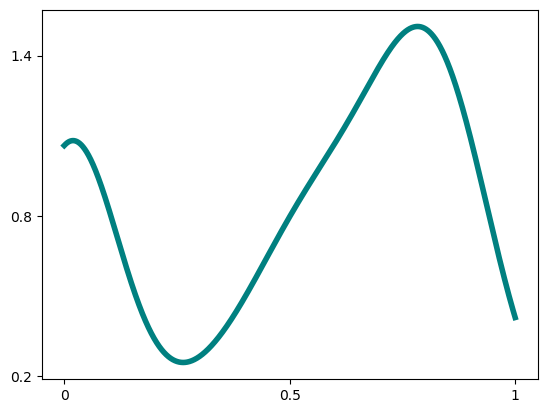}
        \caption{Learned $\hat\eta$ by \cref{alg:latent-GAN}\label{fig:spiral-ours-latent}
        }
    \end{subfigure}%
    \vspace{-0.25em}
    \caption{\label{fig:spiral}
    Multimodal distribution supported on the spiral: true and estimated densities.}
\end{figure*}


\begin{figure*}[t!]
\vspace{-1.3em}
    \centering
    \hspace{-2em}
    \begin{subfigure}[t]{0.33\textwidth}
        \centering        \includegraphics[scale=0.32]{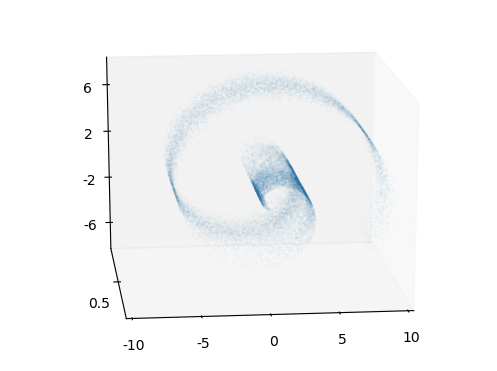}
        \caption{Samples from the target $\rho$}
    \end{subfigure}%
    \begin{subfigure}[t]{0.33\textwidth}
        \centering
        \includegraphics[scale=0.32]{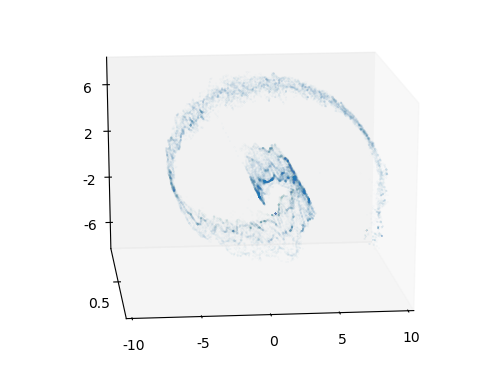}
        \caption{Samples from $\hat T_\push\hat\eta$ from \cref{alg:latent-GAN}}
    \end{subfigure}%
    ~
    \begin{subfigure}[t]{0.33\textwidth}
        \centering
        \includegraphics[scale=0.32]{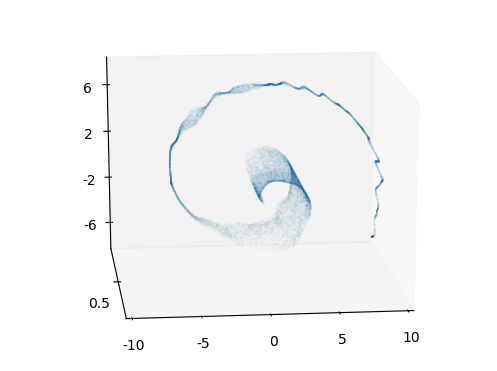}
        \caption{Samples from GAN with fixed $\eta_0$}
    \end{subfigure}%
    \vspace{-0.25em}
    \caption{\label{fig:swissroll}Multimodal distribution supported on the swiss roll: true and estimated distributions.}
    \vspace{-1em}
\end{figure*}

\begin{figure*}[t!]
    \begin{subfigure}[t]{0.38\textwidth}
        \includegraphics[height=0.36\textwidth]{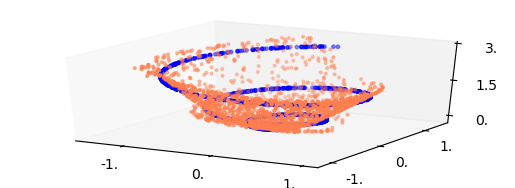}
        \caption{Samples standard GAN}
        \label{fig:samplegauss1}
    \end{subfigure}%
    \begin{subfigure}[t]{0.38\textwidth}
        \includegraphics[height=0.36\textwidth]{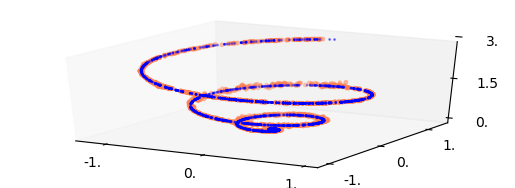}
        \caption{Samples \cref{alg:latent-GAN}}\label{fig:samples_our1}
    \end{subfigure}
    \begin{subfigure}[t]{0.22\textwidth}
    \centering
        \includegraphics[height=0.5\textwidth]{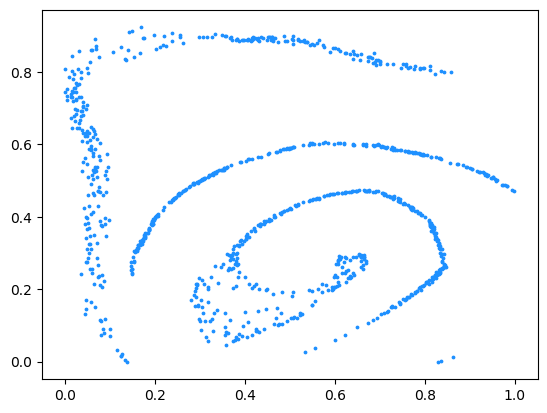}
        \caption{Samples from $\hat\eta$ by \cref{alg:latent-GAN}}\label{fig:latent1}
    \end{subfigure}%

           \begin{subfigure}[t]{0.38\textwidth}
        \includegraphics[height=0.36\textwidth]{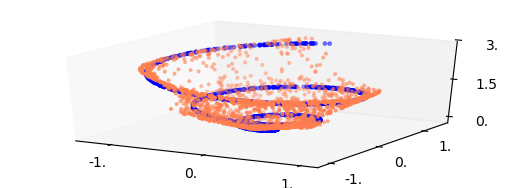}
        \caption{Samples standard GAN}
        \label{fig:samplegauss2}
    \end{subfigure}%
    \begin{subfigure}[t]{0.38\textwidth}
        \includegraphics[height=0.36\textwidth]{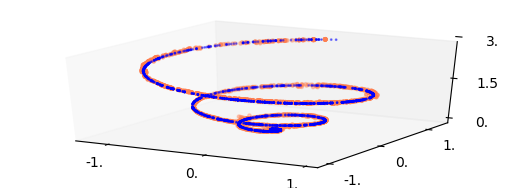}
        \caption{Samples \cref{alg:latent-GAN}}\label{fig:samples_our2}
    \end{subfigure}
    \begin{subfigure}[t]{0.22\textwidth}
    \centering
        \includegraphics[height=0.5\textwidth]{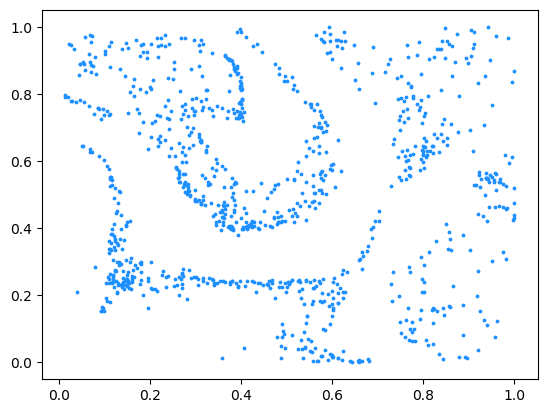}
        \caption{Samples from $\hat\eta$ by \cref{alg:latent-GAN}}\label{fig:latent2}
    \end{subfigure}%
           \caption{results for GAN training with $\T$ space of generators of increasing complexity. (Top row) $\T$ space of {\itshape $2$-layers generators}, (Bottom row) $\T$ space of {\itshape $2$-layers generators}. (First two columns) Samples from the target distribution $\rho$ (blue) and generated samples (orange) for respectively the standard GANs with fixed latent Gaussian distribution (left column) and $\hat T_\push\hat\eta$ learned via \cref{alg:latent-GAN} (central column). (Right column) samples from the latent distribution $\hat\eta$ learned via \cref{alg:latent-GAN}.}\label{fig:helix1}
\end{figure*}

\textbf{From $2$D to $3$D: matching of a 1-dimensional helix.} We considered the task of matching a probability measure $\rho\in\P(\R^3)$ supported on a $1$-dimensional helix-shaped manifold. While the target distribution could be modeled in terms of a latent distribution on the real line and a suitable pushforward, here we consider a model where $\eta\in\P(\R^2)$ is a probability in $2D$ and $T:\R^2\to\R^3$. The goal of this experiment is to qualitatively assess the impact of learning the latent distribution. Analogously to the previous experiments, we compare our algorithm against the standard GAN approach, with latent distribution fixed and equal to the Gaussian measure $\eta_0 = \mathcal{N}([0.0,0.0],I)$. We consider two options for the family $\T$ of candidate generators from $\R^2$ to $\R^3$ with increasing complexity. We aim to show that the joint GAN estimator can efficiently learn the target distribution while the standard GAN algorithm requires a significantly larger space of generators. We show results corresponding to two architectures for learning the generator.

\textit{$2$-layers generator }: we considered $\T$ the space of neural networks from $\R^2$ to $\R^3$ with $2$ hidden layers of dimensions $128,128$ and respectively ReLu and Tanh activation functions. \cref{fig:helix1} (top row) reports the results for GAN training in this setting for the standard GAN (\cref{fig:samplegauss1}) and the proposed estimator (\cref{fig:samples_our1}). When keeping the latent distribution fixed, the generator is not able to match the target. In contrast, by applying \cref{alg:latent-GAN} to learn the latent distribution, part of the complexity of modeling $\rho$ is offloaded to $\eta$ and therefore the final estimator is able to approximately match the target. This is visually reported in \cref{fig:latent1}, which shows a sample from the learned latent distribution $\hat\eta$. It can be noticed that such distribution is significantly different from the Gaussian measure used for standard GAN training. As a result, the generalization gaps (see section \cref{sec:experiments}) are respectively $0.0003$ for \cref{alg:latent-GAN} and $0.0311$ when keeping the latent distribution fixed.

\textit{$4$-layers generator}: we considered $\T$ the space of neural networks from $\R^2$ to $\R^3$ with 4 hidden layers with dimensions $128,512,512, 128$ and ReLu activation functions for layers 1,3 and Tanh for layers 2,4. \cref{fig:helix1} (bottom row) reports the results of GAN training for the standard estimator with fixed latent distribution $\eta_0$ (\cref{fig:samplegauss2}) and the estimator from \cref{alg:latent-GAN} (\cref{fig:samples_our2}). We note that the standard GAN method is still not able to correctly match the target distribution but is incurring in a smaller error. The two methods achieve generalization gaps respectively $0.0003$ for \cref{alg:latent-GAN} and $0.0209$ when keeping the latent distribution fixed. We note that estimator from \cref{alg:latent-GAN} is achieving similar qualitative and quantitative performance to the one obtained using a simpler space of generators $\T$. However we observe a key difference in \cref{fig:latent2}, which reports a sample from the estimated latent distribution $\hat\eta$. Since the generator class is larger, it allows to apply more distortion to latent distributions. As a consequence, the latent distribution can have a less sharp support shape and still realize a good matching.

\begin{figure}[h!]
    \centering
    \includegraphics[height=0.15\textwidth]{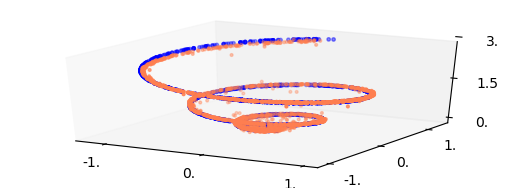}
    \caption{result with standard GAN with $\T$ a space of generators with 6 layers with 512 dimensions.}
    \label{fig:goodgaussian}
\end{figure}

{\itshape Deeper models.} In our experiments, the standard GAN was able to match the target distribution only when $\T$ allowed for deeper networks (\cref{fig:goodgaussian} shows this for $\T$ a space of networks up to 6 layers with 512 neurons each and ReLU activation functions). This is in line with the intuition in \cref{sec:motivation} and the theoretical analysis in this paper: while choosing a fixed latent distribution still allows to recover the target probability, doing so might impose tight requirements on the complexity of the space of generators that needs to be considered.

\section{Conclusions}\label{sec:conclusion}

In this work we studied the role of pushforward maps (generators) and latent distributions within the  Sinkhorn GAN paradigm, from a theoretical perspective. We characterized the learning rates of a GAN estimator in terms of the complexity (i.e. smoothness) of the class of generators. 
We introduced a novel GAN estimator that jointly learns both latent distribution generator, studied its generalization properties 
and proposed a practical algorithm to train it. Future work will focus on two main directions. First, we plan to investigate more empirically oriented questions related to our framework. In particular, we plan to evaluate our approach on large scale real data, to test the limits and benefits of the proposed strategy in practice. 
Secondly, on a more theoretical direction, we plan to extend our analysis to a larger family of adversarial divergences and generator networks.  





{
\bibliographystyle{plain}
\bibliography{biblio}
}

\newpage

\appendix

\crefname{assumption}{Assumption}{Assumptions}
\crefname{equation}{}{}
\Crefname{equation}{Eq.}{Eqs.}
\crefname{figure}{Figure}{Figures}
\crefname{table}{Table}{Tables}
\crefname{section}{Section}{Sections}
\crefname{theorem}{Theorem}{Theorems}
\crefname{proposition}{Proposition}{Propositions}
\crefname{fact}{Fact}{Facts}
\crefname{lemma}{Lemma}{Lemmas}
\crefname{corollary}{Corollary}{Corollaries}
\crefname{example}{Example}{Examples}
\crefname{remark}{Remark}{Remarks}
\crefname{algorithm}{Alg.}{Alg.}
\crefname{enumi}{}{}

\crefname{appendix}{Appendix}{Appendices}

\numberwithin{equation}{section}
\numberwithin{lemma}{section}
\numberwithin{proposition}{section}
\numberwithin{theorem}{section}
\numberwithin{corollary}{section}
\numberwithin{definition}{section}
\numberwithin{algorithm}{section}
\numberwithin{fact}{section}
\numberwithin{remark}{section}

\section*{\Huge\textbf{Supplementary Material}}

The supplementary material is organized as follows:
\begin{itemize}
    \item \cref{app:adversarial-losses} recalls how common loss functions used for GAN training can be formulated as adversarial divergences.
    \item \cref{app:pushforwards} provides details on the examples made in \cref{sec:motivation}. 
    \item \cref{app:technical-results} derives technical results that will be used to prove the main results of this work.
    \item \cref{app:rates} proves the learning rates of the proposed joint GAN estimator proposed in \cref{sec:theory}.
    \item \cref{app:gradient} proves the formula of the gradient of Sinkhorn divergence with respect to network parameters.
    \item \cref{app:sec_experiments} describes the experimental setup and network specification for the empirical evaluation reported in \cref{sec:experiments}. 
\end{itemize}

\section{Adversarial Divergences}\label{app:adversarial-losses}

The notion of {\itshape adversarial divergence} was originally introduced in \cite{liu2017approximation}. 
For completeness, we review how most loss functions used for probability matching within the GANs literature can be formulated as an adversarial divergence in \cref{eq:adversarial-losses}. We recall here the definition in \cref{eq:adversarial-losses} of adversarial divergence over a space $\F$ of functions $F:\X\times\X\to\R$, bewteen two distributions $\mu,\rho\in\P(\X)$ as
\eqal{\label{eq:adversarial-loss-app}
    \msf{d}_\F(\mu,\rho) ~=~ \sup_{F\in\F}~ \int F(x',x) ~ d\mu(x')d\rho(x).
}
Depending on the choice of $\F$ we recover different choices of adversarial divergences as discussed below.

\paragraph{Integral Probability Metrics} One of the most notable examples of adversarial divergences are integral probability metrics (IPM). The  IPM over a space $\F_{0}$ of functions $f:\X\to\R$, between two distribution $\mu,\rho$ is defined as
\eqal{\label{eq:IPM-definition}
    \msf{IPM}_{\F_0}(\mu,\rho) ~ = ~ \sup_{f\in\F_0} ~ \abs{ \int f(x)~d\mu(x) - \int f(x)d\rho(x)}.
}
Clearly, if $\F_0$ is closed under the operation of subtraction, the IPM is an adversarial divergence with $\F$ in \cref{eq:adversarial-loss-app} defined as 
\eqals{
    \F ~=~ \{~F ~|~ F: (x',x) \mapsto f(x') - f(x), \quad f\in\F_0~ \}.
}
Examples include:
\begin{itemize}
\item {\itshape Maximum Mean Discrepancy (MMD).} Here $\F_0$ is the ball of radius $1$ in a Reproducing Kernel Hilbert Space \cite{dziugaite2015training}. 
\item {\itshape $\mu$-Sobolev IPM.} Proposed in \cite{mroueh2017sobolev}. Given a reference measure $\mu\in\P(\X)$, Sobolev-IPM consider $\F_0$ to be
\eqals{
    \F_0~=~\{~f~|~ \mathbb{E}_\mu~\nor{\nabla f(\cdot)}^2\leq 1, \quad f\in W^{1,2}(\X,\mu)~\}
}
with $W^{1,2}(\X,\mu)$ denoting the space of $\mu$-square integrable functions on $\X$ with $\mu$-square integrable first weak derivative. Similarly, {\itshape $\mu$-Fisher-IPM} \cite{mroueh2017fisher} are IPM defined over $\F_0$ the ball of radius $1$ in $L^{2}(\X,\mu)$. 
\item {\itshape $1$-Wasserstein.} The $1$-Wasserstein distance is defined as in \cref{eq:primal_pb} with cost function the Euclidean norm (rather than the squared Euclidean norm considered in this work) and $\eps=0$. The dual fomrmulation is in an IPM loss, with $\F_0$ the ball of radius $1$ in the space of Lipschitz functions (namely all functions with Lipschitz constant less or equal than $1$).
\end{itemize}

\paragraph{$f$-Divergences}  $f$-Divergences are discrepancy measures between two distributions of the form 
\eqal{\label{eq:f-divergence}
    \msf{d}_f(\mu,\rho) ~=~ \int f\Bigg(\frac{d\mu}{d\rho}(x)\Bigg)~d\rho(x)
}
where $d\mu/d\rho$ denotes the Radon-Nykodin derivative of $\mu$ with respect to $\rho$ and $f:\X\to\R\cup\{\infty\}$ is suitable a convex function. By leveraging the notion of Fenchel dual $f^*(y) = \sup_{x} \scal{x}{y} - f(x)$ and the fact that $f^{**} = f$ for convex functios, in \cite{nowozin2016f,liu2017approximation} it was observed that $f$-divergences of the form in \cref{eq:f-divergence} can be written as adversarial divergences with 
\eqals{
    \F_f ~=~ \{ ~ F ~|~F:(x',x)\mapsto g(x') - f^*(g(x)), \quad g\in C_b(\X), ~ {\rm dom}(g)\subset{\rm dom}(f) ~\}
}
with $C_b(\X)$ the set of continuous bounded functions on $\X$.

Examples include (see \cite{nowozin2016f} for more):
\begin{itemize}
    \item {\itshape Kullback-Leibler.} Here $f(x) = x\log x$
    \item {\itshape Jensen-Shannon} divergence. Considered in the orginal GAN paper (up to a constant) \cite{goodfellow2014generative}, here $f(x) = x\log x - (x+1)\log(x+1)$.
    \item {\itshape Squared Hellinger.} here $f(x) = (\sqrt{x} - 1)^2$
\end{itemize}

\paragraph{Entropic Optimal Transport} We conclude this section by reviewing how entropic optimal transport functions can be formulated as adversarial divergences. As observed in \cref{sec:implementation}, the dual probem associated to the definition of $\oteps$ corresponds to \cref{eq:dual_pb}, namely
\eqals{\label{eq:dual_pb-app}
\sup_{u,v\in \cont(\X)}~\int u(x)~d\mu(x) ~+~ \int v(y)~d\rho(y) ~-~\eps \int e^{\frac{u(x) + v(y) - \|x-y\|^2}{\eps}}~d\mu(x)d\rho(y).
}
This problem can be written in the form of adversarial divergence in \cref{eq:adversarial-losses} by taking
\eqals{
    \F ~=~ \{~ F ~ |~ F:(x,y)\mapsto u(x) + v(y) - \eps e^{\frac{u(x)+v(y)-\nor{x-y}^2}{\eps}}, \quad u,v\in C(\X) ~\}. 
}

\section{The Complexity of Pushforward Maps/Generators}\label{app:pushforwards}

We provide here some examples and observations on pushforward maps. Assume that $\Z=\X= \mathbb{R}^n$ and consider $\mu$ to be a Gaussian measure. Consider $\rho\in\prob(\X)$ a target distribution. Since $\mu$ is absolutely continuous with respect to Lebesgue measure $\mathcal{L}^d$, there always exists a measurable map $T$ such that $T_\push \mu = \rho$ (see for example \cite[Thm. 1.33 ]{ambrosio2013user}. However, existence of a pushforward map $T$ does not imply anything on its regularity, unless further assumptions hold on the measures $\mu$ and $\rho$. For example, consider the case where the support of $\rho$ is disconnected: in this case, any pushforward will exhibit discontinuities (since the image through a continuous map of a connected set is always connected). 
In the following we provide a few examples of pushforward map between distributions.

\textbf{1. Book-shifting.} Let $\mu = \chi_{[0,1]}$ and $\rho =  \chi_{[1,2]}$. The maps $T_1, T_2:[0,1]\rightarrow [1,2]$ defined by $T_1(x) = 2-x$,  $T_2(x) = x+1$ are pushforwards from $\mu$ to $\nu$, i.e. ${T_i}_\push \mu = \rho$ for $i=1,2$. \\
\textbf{2. \textbf{Gaussians.}} Let $\mu = \mathcal{N}(m_\mu, \Sigma_\mu)$ and $\rho = \mathcal{N}(m_\rho, \Sigma_\rho)$ be two Gaussians on $\R^d$. The following map 
\eqal{
T: x\rightarrow m_\rho + A(x - m_\mu), \qquad A = \Sigma_\mu^{-\frac{1}{2}} \big(\Sigma_\mu^{\frac{1}{2}}\Sigma_\rho \Sigma_\mu^{\frac{1}{2}}\big)^{\frac{1}{2}}  \Sigma_\mu^{-\frac{1}{2}} 
}
is such that $T_\push \mu = \rho$.\\
Intuitively, when considering measures with the same `structure' (i.e. in the examples above, both uniform or both Gaussian), the `distortion' needed to match the two distribution is mild and this results in very regular pushforward maps, namely linear in the case above. Viceversa, given a measure $\mu$, pushforward via linear maps can target measures with the same structure as $\mu$ only. \\
\paragraph{Example: class of functions and correspondent pushforward measures}
Consider  $\mathcal{T}$ the class of affine  maps on the real line, i.e. 
\begin{equation}
    \mathcal{T}:= \{T^{m,q}: \R \rightarrow \R: \,\, T^{m,q}(x) = mx + q, \,\,\, m,q\in\R, \,\,\ m\neq 0.\}
\end{equation} Let $\rho = r\mathcal{L}^1$ with $r = \mathbbm{1}_{[0,1]}$. Then, all the pushforward measures $T_{\push}\rho$ with $T \in \mathcal{T}$ are of the form 
\begin{equation}
    T^{m,q}_{\push}\rho = \frac{\mathbbm{1}_{[0,1]} (\frac{x - q}{m})}{m} \mathcal{L}^1.
\end{equation}
Hence, all the measures that can be written as pushforward of $\rho$ via maps in $\mathcal{T}$ are uniform measures on intervals in $\R$, namely  $T^{m,q}_{\push}\rho = \frac{1}{m}\mathbbm{1}_{[q, q+m]}$. 

When $\mu$ and the target measure $\rho$ are very different, pushforward maps will be more complex:\\ 
\textbf{3. From uniform to troncated Gaussian} Let $\mu = f \mathbbm{1}_{[-1,1]}$ and $\rho = g  \mathbbm{1}_{[-1.1]} $ with $f(x) =  \sqrt{2/(\pi c)}e^{-\frac{x^2}{2}}$, $c = \textnormal{erf}(1/\sqrt{2})$ and $g(x) = \frac{1}{2}$. Using \cref{eq:pushforward_ODE}, one can show that the map $T$ such that $T_\push \mu = \rho$ is of the form
\begin{equation}\label{eq:push_gauss_uniform}
    T(x) = \sqrt{2}\textnormal{erf}^{-1}\big(  c x \big).
    \end{equation}
Since $\rho$ is not of the form \cref{eq:description_push}, there exists not $T\in \T$ with $\T$ defined in \cref{eq:class_of_linear_push} such that $T_{\push}\rho = \rho$. In order to be able to map $\mu$ into $\rho$, one has to consider a class of function $\T$ large enough to include the function $\textnormal{erf}^{-1}$.\\

From the examples above, it is clear that when given a fixed $\mu$ and a target $\rho$, the regularity (in terms of upper bounds of derivatives ) varies significantly, depending on the properties of $\mu$ and $\rho$. In particular, e measure $\mu$ with a specific structure, may require a pushforward $T$ to have big derivatives, in order to satisfy $T\mu = \rho$. 

\textbf{Example: pushforward from one Gaussian to a mixture of three Gaussians in 1d.}
We computed the pushforward from a Gaussian distribution to a mixture of the Gaussian distributions in 1D with variance $0.05$ and $0.01$ (see \cref{fig:distrib1} and \cref{fig:distrib2}). We computed the pushforward map using a neural network with 5 layers, alternating ReLu and tanh as activation functions. \cref{fig:push1} and \cref{fig:push2} display the graphs of the computed pushforward maps. One can notice that the maps alternate regions with steep derivatives to regions with flat derivatives, needed to distort the mass of the gaussian in order to match the multimodal shape of the target. The steepness is significantly higher in the case where the target has smaller variance (i.e. the three Gaussians are more concentrated leading to areas with a very small amount of mass). 

\textbf{Experimental setup for \cref{fig:example-complexity-push}.}
We used neural network with 1 layer (linear network), 2 layers (with ReLu activation), 5 (up to 256 dimensions) and 7 layers (up to 512 dimensions), alternating ReLu to Tanh activation functions).

\begin{figure*}\label{fig:push_gauss_mixture_gauss}
    \centering
    \begin{subfigure}[t]{0.25\textwidth}
        \centering
        \includegraphics[height=0.8\textwidth]{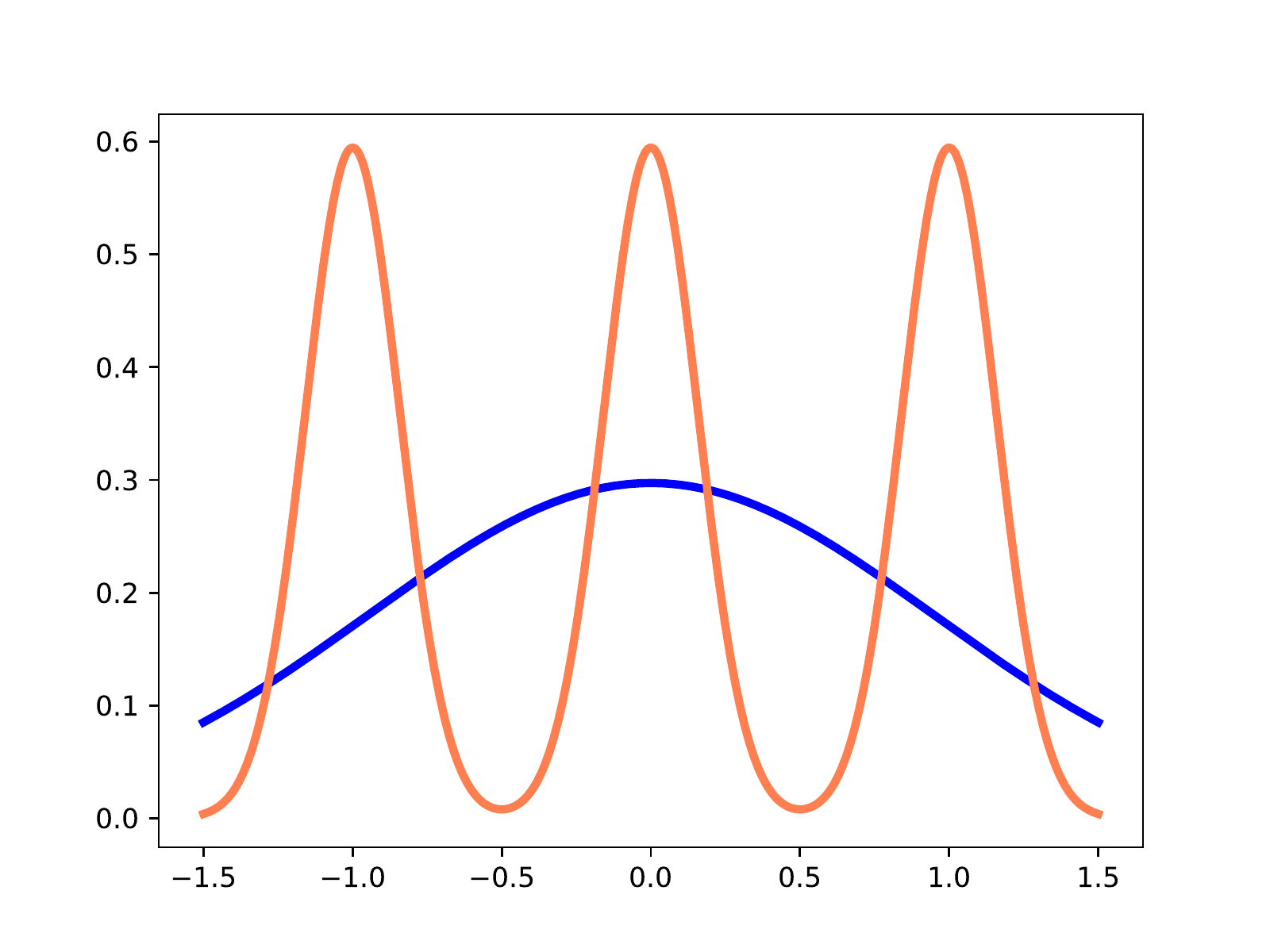}
        \caption{Distributions}
        \label{fig:distrib1}
    \end{subfigure}%
    \begin{subfigure}[t]{0.25\textwidth}
        \centering
        \includegraphics[height=0.8\textwidth]{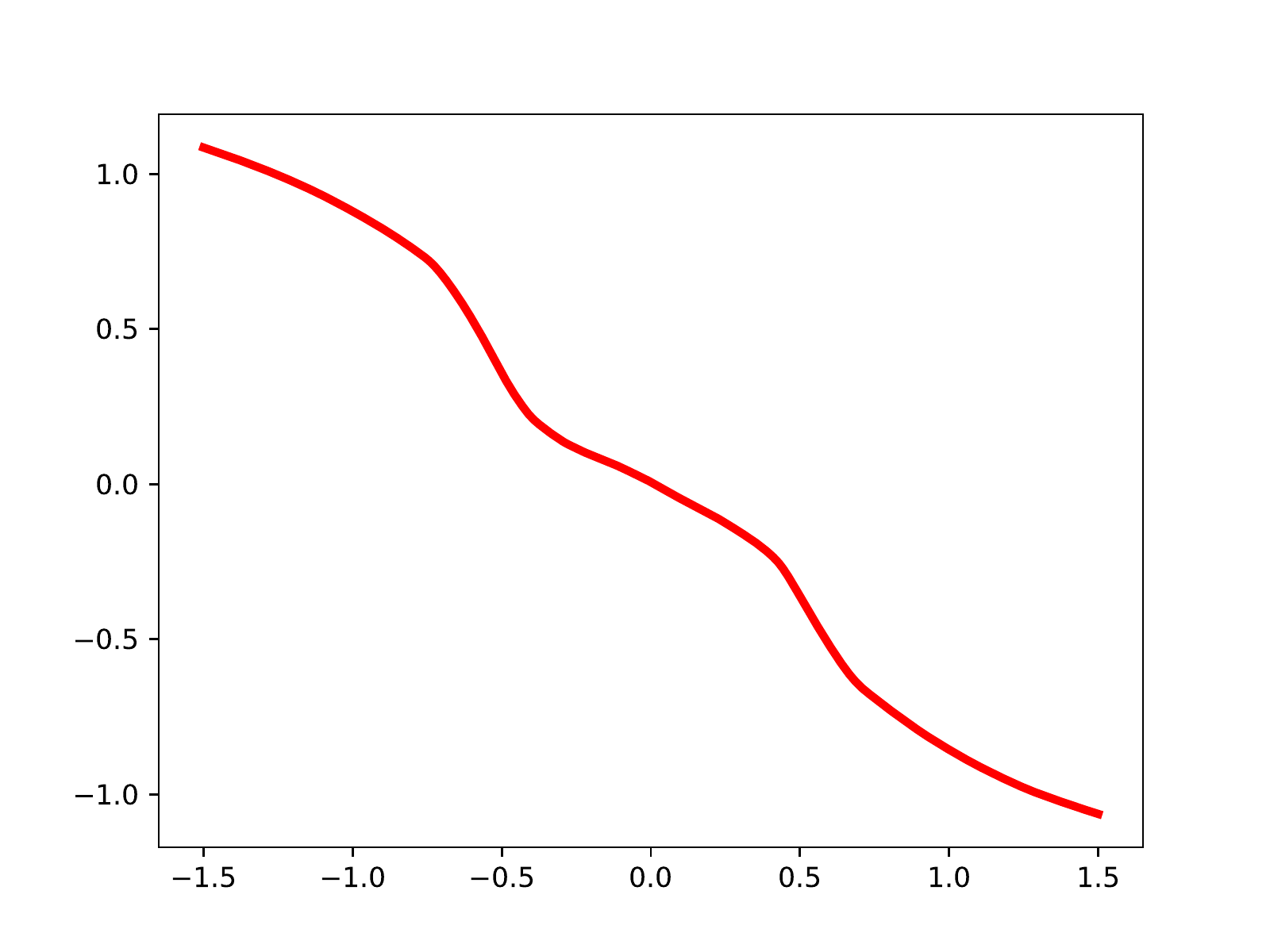}
        \caption{Pushforward map}\label{fig:push1}
    \end{subfigure}\\
    \begin{subfigure}[t]{0.25\textwidth}
        \centering
        \includegraphics[height=0.8\textwidth]{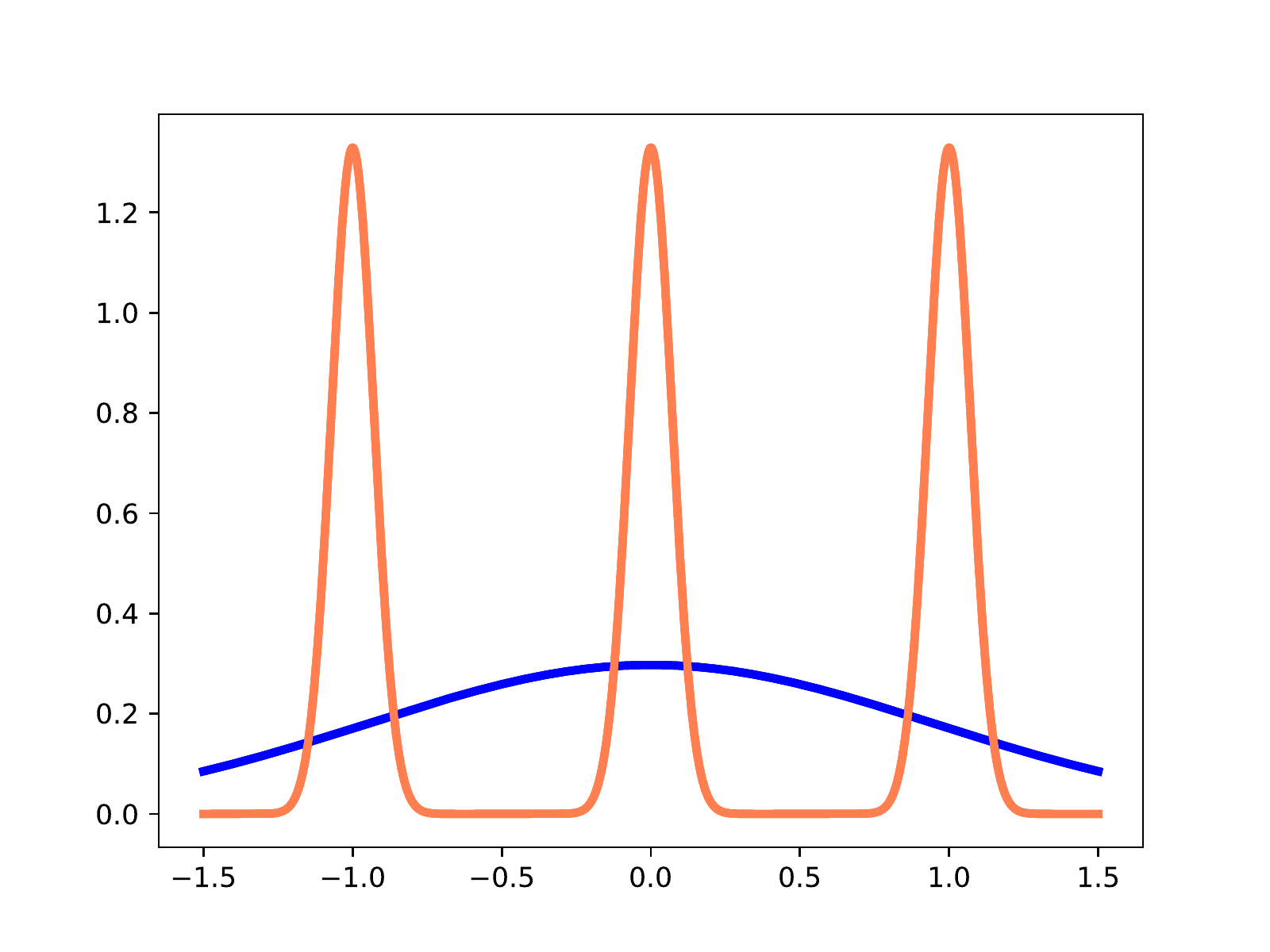}
        \caption{Distributions}\label{fig:distrib2}
    \end{subfigure}%
    \begin{subfigure}[t]{0.25\textwidth}
        \centering
        \includegraphics[height=0.8\textwidth]{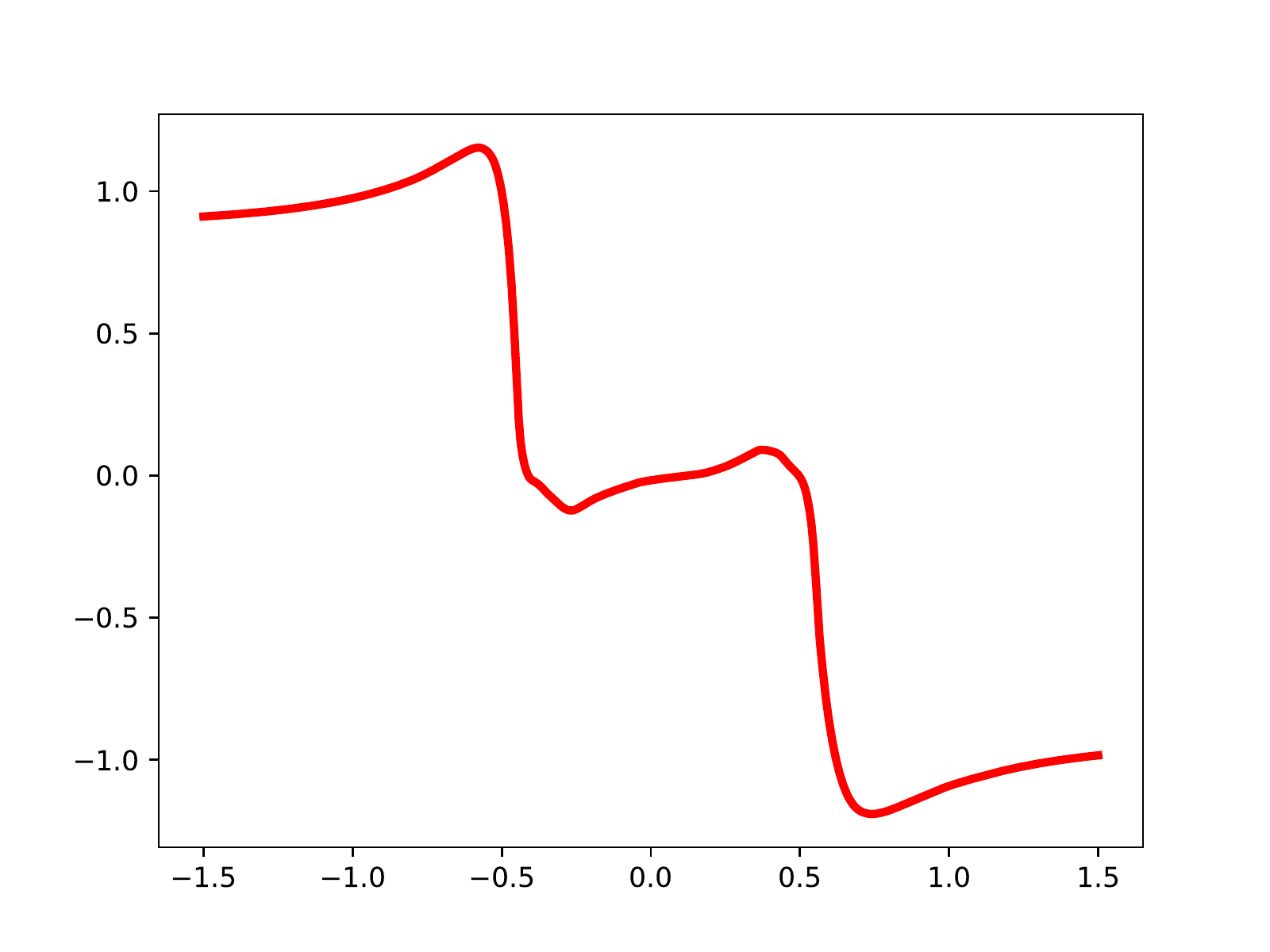}
        \caption{Pushforward map}\label{fig:push2}
    \end{subfigure}
    \caption{Pushforward functions that map $1$ Gaussian to a mixture of three Gaussians. Distributions displayed on the left. Graphs of pushforward maps on the right}.
\end{figure*}

\section{Technical results}
\label{app:technical-results}
We introduce some notation first; in the following, the map $\mathsf{c}:\X \times \X\rightarrow \R$ denotes the Euclidean squared norm, i.e. $\cost(x,y) = \frac{1}{2}\norm{x-y}^2$. Given two maps $V,T:\Z \rightarrow \X$, we used the symbol $\cost$ decorated with subscripts $T$ and $V$ to denote the following: $\cost_{\subt,\subv} :\Z \times \Z \rightarrow \R$ is the function defined by
\eqals{
\cost_{\subt,\subv}(z,w) = \cost(T(z), V(w)), \qquad \textnormal{for all}\,\, z,w\in\Z.
}
Since we need to highlight the dependence on the cost, we will incorporate it in the notation used for Sinkhorn divergence, namely $\mathsf{S}_{\eps, c}$ with $c$ denoting the cost function used. We first recall a straightforward result which links Sinkhorn divergence of pushforward measures with Sinkhorn divergence with modified cost function.  
\begin{lemma}\label{lemma:equivalence_push_change_cost}
 Let $\mu,\nu\in\prob(\Z)$ and $T,V: \Z \rightarrow \X$ be continuous maps. Let $\cost:\X\times \X \rightarrow \R$  and $\mathsf{c}_{T,V}:\Z\times \Z \rightarrow \R$  be as defined above. Then, 
\eqals{
\mathsf{S}_{\eps, \cost} (T_\push \mu, V_\push \nu) = \mathsf{S}_{\eps, \cost_{\subt,\subv}} (\mu, \nu). 
}
\end{lemma}
\begin{proof}
Similarly to $\mathsf{S}_{\eps, \cost_{\subt,\subv}}$, let $\mathsf{OT}_{\eps,\cost}$ be the biased entropic OT problem with cost function $\cost$. 
Let $F(\mu, \nu, u, v, \cost)$ be defined as 
\begin{align*}
F(\alpha, \beta, u, v, \cost) = \int_{\X} u(x)~d\alpha(x) + \int_{\X} v(y)~d\beta(y) -\eps \int e^{\frac{u(x) + v(y) - \cost(x,y)}{\eps}}\,d\alpha(x)d\beta(y).
\end{align*}
By the dual definition of $\mathsf{OT}_{\eps,\cost}$, we have
\begin{align}
  \mathsf{OT}_{\eps,\cost}(T_\push \mu,V_\push \nu) & =\sup_{(u,v)\in \cont(\X)\times \cont(\X) }~ F(T_\push\mu,V_\push\nu,u,v, \cost)\\
  &= \sup_{(u,v)\in \cont(\X)\times \cont(\X) }~ F(\mu,\nu,u\circ T,v\circ V, \cost_{\subt,\subv})\\
    &\label{eq:sup_restricted} = \sup_{(\tilde u,\tilde v)\in (\cont(\X)\circ T)\times(\cont(\X)\circ V) } ~ F(\mu,\nu, \tilde u, \tilde v, \cost_{\subt,\subv}),
\end{align}
by the property of the pushforward and where $\cont(\X)\circ T := \{ f\circ T:\,\, f\in\cont(\X)\}$ and similarly for $\cont(\X)\circ V$. Now, consider 
\eqals{\label{eq:dual_with_diff_cost}
\mathsf{OT}_{\eps,\cost_{\subt,\subv}}(\mu, \nu) = \sup_{(\tilde u, \tilde v) \in \cont(\Z)\times \cont(\Z)} \int_{\Z}\tilde u(z)~d\mu(z) + \int_{\Z}\tilde v(w)~d\nu(w) -\eps \int e^{\frac{u(x) + v(y) - \cost_{\subt,\subv}(x,y)}{\eps}}\,d\mu(z)d\nu(w).
}
We note that the optimal potentials $\tilde u, \tilde v$ of $\mathsf{OT}_{\eps,\cost_{\subt,\subv}}$ have the form \cite{feydy2018interpolating}
\begin{align*}
\tilde u(z) = -\log \int_{\Z} e^{\tilde v(w) - \cost_{\subt, \subv}(z,w)}~d\nu(w).
\end{align*}
Recalling that $\cost_{\subt, \subv}(z,w) = \cost(T(z), V(w))$, we note that $\tilde u$ and $\tilde v$ are functions of the form $u\circ T$ and $v\circ V$. Hence the supremum in \cref{eq:dual_with_diff_cost} can be restricted to be on the sets $\cont(\X)\circ T$ and $\cont(\X)\circ V$. Thus, the quantity in \cref{eq:sup_restricted} equals $\mathsf{OT}_{\eps,\cost_{\subt,\subv}}$. Extending the same argument to the autocorrelation terms, we conclude that $\mathsf{S}_{\eps, \cost} (T_\push \mu, V_\push \nu) = \mathsf{S}_{\eps, \cost_{\subt,\subv}} (\mu, \nu)$ as desired. 
\end{proof}

Before proceeding with the results bounding the potentials of Sinkhorn divergence with cost function $\cost_{\subt, \subv}$, we provide some technical results needed.

\begin{lemma}[Lemma $1$ in {\cite{mena2019statistical}}]\label{lem:bound_on_moment_subg}
If $\mu\in\prob(\Z)$ with $\Z\subset \R^k$ is $\sigma^2$-sub-Gaussian, then
\eqals{
\E_\mu \norm{X}^{2r} \leq (2k\sigma^2)^r r!,
}
for all nonnegative integers $r$. Also, 
\eqals{
\E_\mu e^{v \cdot X}\leq \E_\mu e^{\norm{v}\norm{X}} \leq 2 e^{\frac{k\sigma^2}{2} \norm{v}^2}
}
for any $v \in \R^k$.
\end{lemma}

\begin{lemma}\label{lem:bound_on_moment}
Let $P = T_\push\mu$ with $T:\R^k \rightarrow \R^d$ such that $\norm{T(z)} \leq L\norm{z}$ and $\mu$ $\sigma^2$-sub-Gaussian. Then
\eqals{
\E_P \norm{X}^2 \leq 2 k (L\sigma)^2 \qquad 
\E_P \norm{X} \leq  L\sigma \sqrt{2k} 
}
\end{lemma}
\begin{proof}

\begin{align*}
\E_P\norm{X}^2 = &\int_{\X}\norm{x}^2~dP(x) = \int_{\X}\norm{x}^2~d(T_\push\mu)(x)\\
= &\label{eq:last_equation}\int_{\Z} \norm{T(z)}^2\,d\mu(z)\leq  L^2\int_{\Z} \norm{z}^2\,d\mu(z).
\end{align*}
Since $\mu$ is $\sigma^2$-sub-Gaussian, we have that 
\eqals{ \E_\mu
\frac{\norm{z}^{2r}}{(2k\sigma^2)^r r!}\leq \E_\mu e^{\frac{\norm{z}^2}{2k\sigma^2}}-1 \leq 1.
}
Thus, $L^2\E_\mu\norm{z}^2 \leq 2 L^2 k\sigma^2$ and combining this with \cref{eq:last_equation}  we obtain 
\eqals{
\E_P\norm{X}^2  \leq L^2 2k\sigma^2.
}
An easy application of Jensen inequality yields the bound for $\E_P\norm{X}$.
\end{proof}

\begin{lemma}[\textnormal{\textbf{Bounds on potentials with changed cost function}}]\label{lem:bounds_pot1}
Let $\mu, \nu \in \prob(\Z)$ be $\sigma^2$-sub-Gaussian measures and consider Sinkhorn divergence with $\eps=1$ and $\cost(x,y) = \frac{1}{2}\norm{T(x) - V(y)}$ as cost function, where $T,V:\Z\rightarrow \X$ are such that $\norm{K(z)}\leq L \norm{z}$  for $K=T,V$. Let $(f,g)$ denote a pair of optimal potentials. Then,
\eqals{
-1 -k(L\sigma)^2 \big( 1 + \frac{1}{2}(\norm{T(x)} + \sqrt{2k}L\sigma)\big)^2 \leq f(x) \leq \frac{1}{2}\big(\norm{T(x)} + \sqrt{2k}L\sigma\big)^2\\
-1 -k(L\sigma)^2 \big( 1 + \frac{1}{2}(\norm{V(y)} + \sqrt{2k}L\sigma)\big)^2 \leq g(y) \leq \frac{1}{2}\big(\norm{V(y)} + \sqrt{2k}L\sigma\big)^2. 
}
\end{lemma}

\begin{proof}
Let $(f_0, g_0)$ any pair of optimal potentials. Since potentials are defined up to constant, we assume as in \cite{mena2019statistical} that $\E_\mu f_0 = \E_\nu g_0 = \frac{1}{2}\fS(\mu,\nu)$. We define 
\eqals{
f(x) = -\log \int_{\X} e^{g_0(y) - \frac{1}{2} \norm{T(x)-V(y)}^2}~d\nu(y)\\
g(y) = \log \int_{\X} e^{f(x) - \frac{1}{2}\norm{T(x)-V(y)}^2}~d \mu(x),
}
for any $x,y\in\X$.  Once we have proved that they are well defined and we have shown the desired lower and upper bound, the proof that they are optimal potentials is exactly the same as in \cite[Prop 6]{mena2019statistical}.
By Jensen inequality
\begin{align*}
&g_0(y) = -\log \int_\X e^{f_0(x) - \frac{1}{2}\norm{T(x)-V(y)}^2}~d\mu(x) \\
&\leq -\E_\mu f_0(x) + \frac{1}{2}\E_\mu\norm{T(X)-V(y)}^2 \leq \frac{1}{2}\E_\mu\norm{T(X)-V(y)}^2.
\end{align*}
Therefore 
\begin{align*}
e^{g_0(y) - \frac{1}{2}\norm{T(x)-V(y)}^2} \leq e^{\frac{1}{2}\E_\mu\norm{T(X) - V(y)}^2 - \frac{1}{2}\norm{T(x)-V(y)}}.
\end{align*}
Expanding the squares we have 
\begin{align*}
& \frac{1}{2}\E_\mu\norm{T(X)-V(y)}^2 - \frac{1}{2}\norm{T(x)-V(y)}^2\\
&= \frac{1}{2}\E_\mu \norm{T(X)}^2 + \frac{1}{2}\norm{V(y)}^2 - \E_\mu\scal{T(X)}{V(y)} - \frac{1}{2}\norm{T(x)}^2 - \frac{1}{2}\norm{V(y)}^2 - \scal{T(x)}{V(y)}.
\end{align*}
Using \cref{lem:bound_on_moment}, we have 
\begin{align*}
 \frac{1}{2}\E_\mu\norm{T(X)-V(y)}^2  - \frac{1}{2}\norm{T(x)-V(y)}^2 & \leq 
L^2  k \sigma^2 +\E_\mu \norm{T(X)}\norm{V(y)} + \norm{T(x)}\norm{V(y)}\\
 & \leq k (L\sigma)^2 +\norm{V(y)}(\norm{T(x)} +\sqrt{2k}L\sigma).
\end{align*}

 Now, with elementary computations and using $\sigma^2$-sub-Gaussianity of $\nu$, we have 
\begin{align*}
 \int_{\Z} e^{k(L\sigma)^2 +\norm{V(y)}( \norm{T(x)}+ \sqrt{2k}L\sigma)}\,d\nu(y) \leq 2 e^{k(L\sigma)^2 \big( 1 + \frac{1}{2}(\norm{T(x)} + \sqrt{2k}L\sigma)\big)^2}.
 \end{align*}

We have shown that 
\begin{align*}
\int e^{g_0(y) - \frac{1}{2}\norm{T(x)-V(y)}^2} ~d\nu(y) \leq   2 e^{k(L\sigma)^2 \big( 1 + \frac{1}{2}(\norm{T(x)} + \sqrt{2k}L\sigma)\big)^2}.
 \end{align*}
Now,
\begin{align*}
f(x) =& -\log\int e^{g_0(y) - \frac{1}{2}\norm{T(x)-V(y)}^2}~d\nu(y) \\
\geq & -\log( 2 e^{k(L\sigma)^2 \big( 1 + \frac{1}{2}(\norm{T(x)} + \sqrt{2k}L\sigma)\big)^2} )\\
\geq & -1 -k(L\sigma)^2 \big( 1 + \frac{1}{2}(\norm{T(x)} + \sqrt{2k}L\sigma)\big)^2,
 \end{align*}
proving the desired lower bound.
We now study the upper bound for $f$: 
\begin{align*}
f(x) &= - \log \int e^{g_0(y)-\frac{1}{2}\norm{T(x)-V(y)}^2}~d\nu(y) \\
& \leq \int -\log e^{g_0(y)-\frac{1}{2}\norm{T(x)-V(y)}^2}~d\nu(y)\\
& = -\int g_0(y)~d\nu(y) + \int \frac{1}{2}\norm{T(x)-V(y)}^2~ d \nu(y)\\
& \leq \int \frac{1}{2}\norm{T(x)-V(y)}^2~ d \nu(y).
 \end{align*}
Developing the square and bounding $\E_\nu \norm{V(Y)}^2$ and $\E_\nu \norm{V(Y)}$, with \cref{lem:bound_on_moment}, we have 
\eqals{
f(x) \leq \frac{1}{2}\big(\norm{T(x)} + \sqrt{2k}L\sigma\big)^2.
}
With the exact same reasoning one can derive the analogous bound for $g$.
\end{proof}
Note that in terms of $\norm{x}$ (and not $\norm{T(x)}$) the derived bounds become
\eqals{\label{eq:bp1}
-1 -k(L\sigma)^2 \big( 1 + \frac{1}{2}L^2(\norm{x} + \sqrt{2k}\sigma)\big)^2 \leq f(x) \leq \frac{1}{2}L^2\big(\norm{x} + \sqrt{2k}\sigma\big)^2\\\label{eq:bp2}
-1 -k(L\sigma)^2 \big( 1 + \frac{1}{2}L^2(\norm{y} + \sqrt{2k}\sigma)\big)^2 \leq g(y) \leq \frac{1}{2}L^2\big(\norm{y} + \sqrt{2k}\sigma\big)^2. 
}
Therefore we have the following result:
\begin{lemma}\label{lemma:bound1_neat}
In the assumptions of \cref{lem:bounds_pot1} we have 
\eqals{\label{eq:bound_potentials_T}
\abs{f(z)} \leq C_{k} \begin{cases}
1 + (L\sigma)^4 \qquad &\textnormal{if }\norm{z}\leq \sqrt{k}\sigma\\
1 +(1 + (L\sigma)^2) L^2 \norm{z}^2 &\textnormal{if }\norm{z}> \sqrt{k}\sigma.
\end{cases}
}
 where $C_k$ is a constant depending only on $k$. 
\end{lemma}

\begin{proof}
This is an immediate consequence of  \cref{eq:bp1} and \cref{eq:bp2}.
\end{proof}

\begin{lemma}[\textnormal{\textbf{Bounds on  derivatives of potentials with changed cost function}}]\label{lem:bounds_pot2}
Let $\mu, \nu \in \prob(\Z)$ be $\sigma^2$-sub-Gaussian measures and consider Sinkhorn divergence with $\eps=1$ and $\cost(x,y) = \frac{1}{2}\norm{T(x) - V(y)}$ as cost function, where $T,V:\Z\rightarrow \X$ are such that $\norm{K(z)}\leq L \norm{z}$  for $K=T,V$. Also, assume that for any multiindex $\alpha$ with length at most $\abs{\alpha} \leq \lfloor k/2 \rfloor + 1$, $\norm{D^\alpha T(x)}_\infty \leq \tau.$ Let $(f,g)$ denote a pair of optimal potentials. Then,
\eqals{\label{eq:bound_derivative_potentials_T}
\abs{D^\alpha\big(f(\cdot) - \frac{1}{2}\norm{T(\cdot)}^2\big)(z)} \leq C_{k, \abs{\alpha}}
\begin{cases}
(\tau L\sigma)^{\abs{\alpha}} (1 + ((L\sigma)^2 + L\sigma)^{\abs{\alpha}}) \quad &\textnormal{if }\norm{z} \leq \sqrt{k\sigma}\\
(\tau L\sigma)^{\abs{\alpha}} (1 + (\sqrt{(L\sigma)\norm{z}} +(L^2 \sigma)\norm{z} )^{\abs{\alpha}}) \quad &\textnormal{if }\norm{z} > \sqrt{k\sigma}.
\end{cases}
}
\end{lemma}
\begin{proof}
Potentials $(f,g)$ are chosen as in \cref{lem:bounds_pot1}. For convenience, set $\bar{f}$ the fuction defined by $\bar{f}(x) = f(x) - \frac{1}{2}\norm{T(x)}^2$. The goal in now to bound the derivatives of $\bar f$, namely $\abs{D^{\alpha} \bar f(x)}$.
Note that 
\begin{align*}
D^\alpha \bar f(x) = - D^\alpha \log (e^{-\bar f(x)}) &= - D^\alpha (\log (\int e^{g(y) - \frac{1}{2}\norm{T(x) - V(y)}^2 + \frac{1}{2}\norm{T(x)}^2}\,d\nu(y))\\
& =  - D^\alpha (\log (\int e^{g(y) - \frac{1}{2}\norm{V(y)}^2 + \scal{V(y)}{T(x)}}\,d\nu(y))\\
& = \frac{D^\alpha \int e^{g(y) - \frac{1}{2}\norm{V(y)}^2 + \scal{V(y)}{T(x)}}\,d\nu(y))}{\int e^{g(y) - \frac{1}{2}\norm{V(y)}^2 + \scal{V(y)}{T(x)}}\,d\nu(y))}.
\end{align*}

Using Faa' di Bruno formula, we have 
\begin{align*}
D^\alpha \int e^{g(y) - \frac{1}{2}\norm{V(y)}^2 + \scal{V(y)}{T(x)}}\,d\nu(y)) = \int \mathbf{P}([\scal{D^j T(x)}{V(y)}]_{j\leq \abs{\alpha}}) e^{g(y) - \frac{1}{2} \norm{V(y)}^2 - \scal{T(x)}{V(y)}} \,d\nu(y),
\end{align*}
where $\mathbf{P}$ is a polynomial of degree $\abs{\alpha}$. In order to bound $\abs{D^\alpha \bar f(x)}$, we have to bound the quantity 
\begin{align*}
A(x) := \frac{\int \mathbf{P}([\scal{D^j T(x)}{V(y)}]_{j\leq \abs{\alpha}}) e^{g(y) - \frac{1}{2} \norm{V(y)}^2 - \scal{T(x)}{V(y)}} \,d\nu(y)}{\int e^{g(y) - \frac{1}{2}\norm{V(y)}^2 + \scal{V(y)}{T(x)}}\,d\nu(y))}.
\end{align*}
To simplify the notation, set
\begin{align*}
E(g,V)(y) : = e^{g(y) - \frac{1}{2} \norm{V(y)}^2 - \scal{T(x)}{V(y)}}\quad \textnormal{and} \quad \mathsf{B}: =\int e^{g(y) - \frac{1}{2}\norm{V(y)}^2 + \scal{V(y)}{T(x)}}\,d\nu(y)).
\end{align*}
Now, set $\mathcal{D}:=\{y: \,\norm{y} \leq \mathsf{h} \}$ with $\mathsf{h}$ to be chosen later, and $\mathcal{D}^c$ the complementary set. We split the quantity $A(x)$ as follows:
\begin{align*}
A(x) = A_1(x) + A_2(x)
\end{align*}
with 
\begin{align*}
&A_1(x) = \int \mathbbm{1}_{\mathcal{D}} \mathbf{P}([\scal{D^j T(x)}{V(y)}]_{j\leq \abs{\alpha}}) E(g,V)(y)\,d\nu(y) / \mathsf{B},\\
&A_2(x) = \int \mathbbm{1}_{\mathcal{D}^c} \mathbf{P}([\scal{D^j T(x)}{V(y)}]_{j\leq \abs{\alpha}}) E(g,V)(y)\,d\nu(y) / \mathsf{B}.
\end{align*}
We bound the two terms separately: note that on $\mathcal{D}$ we have $\norm{V(y)} \leq L \norm{y} \leq L \mathsf{h}$ and hence
\begin{align*}
A_1(x) \leq \sup_{x} \mathbf{P}([\scal{\norm{D^j T(x)}}{L \mathsf{h}}]_{j\leq \abs{\alpha}})\leq C_{\abs{\alpha}} \tau^{\abs{\alpha}}(L\mathsf{h})^{\abs{\alpha}},
\end{align*}
since we can assume without lost of generality that $\tau \geq 1$ and $L\geq 1$. 
As for $A_2$, we proceed as follows. First, applying \cref{lem:bounds_pot1} we have that 
\begin{align*}
\frac{1}{\mathsf{B}} = \Big(\int E(g,V)(y)\,d\nu(y)\Big)^{-1} = e^{\bar f(x)} \leq e^{-\frac{1}{2}\norm{T(x)}^2 + f(x)} \leq e^{k (L \sigma)^2 + \norm{T(x)}\sqrt{2k} L\sigma} 
\end{align*}
and
\begin{align*}
e^{g(y) - \frac{1}{2}\norm{V(y)}^2}\leq e^{k (L \sigma)^2 + \norm{V(y)}\sqrt{2k} L\sigma}.
\end{align*}
Using these inequalities, we obtain
\begin{align*}
A_2 & \leq c_1\int \mathbbm{1}_{\mathcal{D}^c} \mathbf{P}([\scal{D^j T(x)}{V(y)}]_{j\leq \abs{\alpha}}) e^{\norm{V(y)} \sqrt{2k} L\sigma + \scal{T(x)}{V(y)}} \,d\nu(y)\\
 & \leq c_1 \int \mathbbm{1}_{\mathcal{D}^c} \mathbf{P}([\scal{D^j T(x)}{V(y)}]_{j\leq \abs{\alpha}}) e^{\norm{V(y)}(\sqrt{2k} L\sigma + \norm{T(x)})} \,d\nu(y)\\
 & \leq C_{\abs{\alpha}} c_1 \tau^{\abs{\alpha}} L^{\abs{\alpha}} \Big( \int \mathbbm{1}_{\mathcal{D}^c}\norm{y}^{2\abs{\alpha}} \,d\nu(y)\Big)^{1/2} \Big( \int \mathbbm{1}_{\mathcal{D}^c} e^{2\norm{V(y)}(\sqrt{2k} L\sigma + \norm{T(x)})} \,d\nu(y) \Big)^{1/2},
\end{align*}
with $c_1 = e^{2k (L \sigma)^2 + \norm{T(x)}\sqrt{2k}\sigma L }$.
Now, 
\begin{align*}
\Big( \int \mathbbm{1}_{\mathcal{D}^c}\norm{y}^{2\abs{\alpha}} \,d\nu(y)\Big)^{1/2} \leq e^{\frac{-\mathsf{h}^2}{8 k \sigma^2}} \Big(\int\mathbbm{1}_{\mathcal{D}^c} e^{\frac{\norm{y}^2}{4 k \sigma^2}} \norm{y}^{2\abs{\alpha}}\,d\nu(y) \Big)^{1/2},
\end{align*}
and applying Young inequality, the subgaussianity of $\nu$ and \cref{lem:bound_on_moment_subg}, we have
\begin{align*}
\Big( \int \mathbbm{1}_{\mathcal{D}^c}\norm{y}^{2\abs{\alpha}} \,d\nu(y)\Big)^{1/2} \leq e^{\frac{-\mathsf{h}^2}{8 k \sigma^2}}\sqrt{2}(2\abs{\alpha})!^{1/4} (\sqrt{2k}\sigma)^{\abs{\alpha}}.
\end{align*}
Also, 
\begin{align*}
\Big( \int \mathbbm{1}_{\mathcal{D}^c} e^{2\norm{V(y)}(\sqrt{2k} L\sigma + \norm{T(x)})} \,d\nu(y) \Big)^{1/2}
\leq 2 e^{2L^2(\sqrt{2k} L\sigma + \norm{T(x)})^2 k\sigma^2}.
\end{align*}
Choosing $\mathsf{h}^2\geq C_{k,\abs{\alpha}}\sigma^2 ((L\sigma)^2+ (L\sigma)^4)$ if $\norm{x} \leq \sqrt{k\sigma}$ and $\mathsf{h}^2\geq C_{k,\abs{\alpha}}\sigma^2 (\sigma L \norm{x} + \sigma^2 L^4 \norm{x}^2)$ if $\norm{x} > \sqrt{k\sigma}$ for a sufficiently large constant $C_{k,\abs{\alpha}}$, then we have that
\begin{align*}
A_2 \leq C_{k,\abs{\alpha}} (\tau \sigma L)^{\abs{\alpha}}.
\end{align*}
Combining this with the bound on $A_1$, we obtain:
\begin{align*}
A(x) \leq C_{k,\abs{\alpha}} (\tau L\sigma)^{\abs{\alpha}} (1 + ((L\sigma)^2 + L\sigma)^{\abs{\alpha}}) \qquad \textnormal{if }\norm{x} \leq \sqrt{k\sigma},\
\end{align*}
and
\begin{align*}
A(x) \leq C_{k,\abs{\alpha}} (\tau L\sigma)^{\abs{\alpha}} (1 + (\sqrt{(L\sigma)\norm{x}} +(L^2 \sigma)\norm{x} )^{\abs{\alpha}}) \qquad \textnormal{if }\norm{x} > \sqrt{k\sigma}.
\end{align*}
\end{proof}

\begin{lemma}\label{lemma:bound2_neat}
In the assumptions of \cref{lem:bounds_pot2} we have, for any multi-index $\alpha$, 
\eqals{\label{eq:bound_potentials_T_der}
\abs{D^{\alpha}f(z)} \leq C_{k, \abs{\alpha}}\begin{cases}  \tau L \norm{z} + (\tau L\sigma)^{\abs{\alpha}} (1 + (L\sigma)^{2\abs{\alpha}}) 
  \qquad &\textnormal{if }\norm{z}\leq \sqrt{k}\sigma\\
  \tau L \norm{z} +  (\tau L\sigma)^{\abs{\alpha}} (1 + (L^2 \sigma\norm{z})^{\abs{\alpha}})
  &\textnormal{if }\norm{z}> \sqrt{k}\sigma.
\end{cases}
}
 where $C_k$ is a constant depending only on $k$. 
\end{lemma}
\begin{proof}
The proof follows by easy manipulation of the terms in \cref{eq:bound_derivative_potentials_T}.
\end{proof}

We present the formal version of \cref{lem:informal-uniform-upper-bound-space-of-functions}.
\begin{lemma}\label{lem:formal-uniform-upper-bound-space-of-functions}
Let  $\F_{\sigma,\tau,L}$ be the space of functions satisfying inequalities \cref{eq:bound_potentials_T} and \cref{eq:bound_potentials_T_der}.
Let $\eta,\nu_1,\nu_2\in\mathcal{G}_\sigma(\Z)$ and $T,T'\in \T$ with $\T$ as in \cref{thm:main-rates}. Let $\mathsf{S}$ denote Sinkhorn divergence with $\eps=1$.  
Then,
\eqal{\label{eq:upper-bound-informal-lemma1}
    \abs{\mathsf{S}(T_\push\eta,T'_\push\nu_1)-\mathsf{S}(T_\push\eta, T'_\push\nu_2)}~\leq~ \sup_{f\in\F_{\sigma,\tau,L}}\abs{\int f(z)~d\nu_1 - \int f(z)~d\nu_2(z)}.
}
\end{lemma}
\begin{proof}
The proof follows exactly the same lines as the proof of \cite[Cor 2]{mena2019statistical} with this variant: the set $\F_\sigma$ is replaced by the set $\F_{\sigma,\tau,L}$, thanks to our estimates on the potentials in the bounds \cref{lemma:bound1_neat} and \cref{lemma:bound2_neat}.
\end{proof}

\begin{remark}
Define the set $\F^s$ to be the set of functions satisfying 
\eqals{
&\abs{f(x)} \leq C_{s,k}(1 + \norm{x}^2)\\
&\abs{D^\alpha f(x)} \leq C_{s,k}(1 + \norm{x}^s)\quad \abs{\alpha} \leq s. 
}
Note that for a sufficiently big constant $C_{s,k}$, for any $f\in\F_{\sigma,\tau,L}$ the function $\frac{1}{(\tau L)^{s}+(\sigma L)^{3s}\tau^s} f$ belongs to $\F^s$.
\end{remark}

\begin{theorem}\label{thm:sample_complexity}.
With the same notation as above, the following holds
\eqals{\label{eq:boundwithT}
  \E \sup_{T\in\T, \eta\in\hh}\abs{\sink({T}_{\#}\eta, \rho_n) -  \sink({T}_{\#}\eta, \rho)} ~\leq~ \frac{\msf c(\tau,L,\sigma,k)}{\sqrt{n}}
}
where $\msf c(\tau,L,\sigma,k) ~=~  \msf C_{k}~ (\tau L)^{\lceil \frac{k}{2} \rceil+1}\big(1 + L^{k+2}(1 + \sigma^{\lceil \frac{5k}{2} \rceil + 6 } ) ~\eps^{-\lceil \frac{5k}{4}  \rceil - 3 }\big)$
 with $\msf C_k$ a constant depending only on the latent space dimension $k$.

\end{theorem}
\begin{proof}
We first set $\eps=1$ and consider $\mathsf{S}$, and then obtain the bound for the general case.
For a given $T\in \T$ and $\eta\in\hh$, by  \cite[Prop 2]{mena2019statistical}, and \cref{lemma:equivalence_push_change_cost} we have
\eqals{
&\abs{\mathsf{S}({T}_{\#}\eta, \rho_n) -  \mathsf{S}({T}_{\#}\eta, \rho)}=\abs{\mathsf{S}({T}_{\#}\eta, T^*_\push \eta^*_n) -  \mathsf{S}({T}_{\#}\eta,  T^*_\push \eta^*)}\\
= &\abs{\mathsf{S}_{\cost_{\subt, \subt^*}} (\eta, \eta^*_n) -
\mathsf{S}_{ \cost_{\subt, \subt^*}} (\eta, \eta^*)}\leq 
\sup_{f\in\mathcal{F}_{\sigma, L, \tau}} \abs{\int_{\Z} f\,d (\eta^*_n - \eta^*)}.
}
Note that the set $\mathcal{F}_{\sigma, L, \tau}$ is independent of the specific $\eta,\eta^*$ and $T,T^*$ and depends only on the properties of the classes that we consider, i.e. $\sigma^2$-sub-gaussianity and boundedness $L$ and smoothness $\tau$ for functions in $\T$. Thus, we can take the supremum in the left hand side over $\eta\in\hh$ and $T\in\T$. Hence, we can take the supremum over $T$ and $\eta$ on the left hand side, namely
\eqals{
\sup_{T\in\T, \eta\in\hh}\abs{\mathsf{S}({T}_{\#}\eta, \rho_n) -  \mathsf{S}({T}_{\#}\eta, \rho)}\leq 
\sup_{f\in\mathcal{F}_{\sigma, L, \tau}} \abs{\int_{\Z} f\,d (\eta^*_n - \eta^*)}.
}
From now on,  recalling that  for any $f\in\F_{\sigma,\tau,L}$ the function $\frac{1}{(\tau L)^{s}+(\sigma L)^{3s}\tau^s} f$ belongs to $\F^s$, from now on the proof is identical to the proof of \cite[Thm. 2]{mena2019statistical} and it leads to the following bound for any $\eps$:
\eqals{
  \E \sup_{T\in\T, \eta\in\hh}\abs{\sink({T}_{\#}\eta, \rho_n) -  \sink({T}_{\#}\eta, \rho)} ~\leq~ \frac{\msf c(\tau,L,\sigma,k)}{\sqrt{n}}
}
where $\msf c(\tau,L,\sigma,k) ~=~  \msf C_{k}~ (\tau L)^{\lceil \frac{k}{2} \rceil+1}\big(1 + L^{k+2}(1 + \sigma^{\lceil \frac{5k}{2} \rceil + 6 } ) ~\eps^{-\lceil \frac{5k}{4}  \rceil - 3 }\big)$
 with $\msf C_k$ a constant depending only on the latent space dimension $k$.
\end{proof}
\section{Learning Rates}\label{app:rates}
We provide here a formal statement of \cref{thm:main-rates}. 
\begin{theorem}
Let $\Z\subset\R^k$, $\X\subset\R^d$ and $\rho = T^*_\push \eta^*$ with  $T^*\in\T \subset C_{\tau,L}^{\lceil k/2\rceil + 1}(\Z,\X)$ and $\eta^*\in\hh \subset \mathcal{G}_\sigma(\Z)$. Let $(\hat T,\hat \eta)$ satisfy \cref{eq:erm-joint-gan} with $\msf{d}_\F = \sink$ and $\rho_n$ a sample of $n$ i.i.d. points from $\rho$. Then,
\eqals{
    \mathbb{E}~\sink(\hat T_\push\hat\eta,\rho) ~\leq~ \frac{\msf c(\tau,L,\sigma,k)}{\sqrt{n}}
}
where $\msf c(\tau,L,\sigma,k) ~=~  \msf C_{k}~ (\tau L)^{\lceil \frac{k}{2} \rceil+1}\big(1 + L^{k+2}(1 + \sigma^{\lceil \frac{5k}{2} \rceil + 6 } ) ~\eps^{-\lceil \frac{5k}{4}  \rceil - 3 }\big)$
 with $\msf C_k$ a constant depending only on the latent space dimension $k$ and where the expectation is taken with respect to $\rho_n$.
\end{theorem}

\begin{proof} We decompose the error as follows
\begin{align}\label{eq:part1}
         \sink(\hat T_\push \hat \eta, \rho) - \sink(T^*_{\push}\eta^*, \rho)&  = A_1 + A_2+ A_3
\end{align}
where
\begin{align}
    & A_1 =   \sink(\hat T_{\push}\hat \eta, \rho) - \sink(\hat T_{\push}\hat \eta, \rho_n)\\
    & A_2 =  \sink(\hat T_{\push}\hat \eta, \rho_n) -  \sink({T}^*_{\push}\eta^*, \rho_n) \\
    & A_3 =     \sink({T}^*_{\push}\eta^*, \rho_n) - \sink(T^*_{\push}\eta^*, \rho).
\end{align}
Note that by optimality of $\hat T$ and $\hat \eta$, $A_2\leq 0$.
Now, 
\eqals{\label{eq:part2}
   A_1 + A_3 \leq 2 \sup_{T\in\T,\eta\in\hh} \Big[\sink({T}_{\push}\eta, \rho_n) -  \sink({T}_{\push}\eta, \rho)\Big].
}
 Applying \cref{thm:sample_complexity} and combining it with \cref{eq:part1} and \cref{eq:part2} yields the desired result.
\end{proof}

\subsection{Perturbation case}\label{app:perturbation}

We conclude this section extending \cref{thm:main-rates}  to the case where the GAN model is accurate up to a perturbation of the pushforward measure in terms of a subgaussian distribution.

\begin{lemma}[Pushforward of a sub-Gaussian measure]\label{lem:pushforward-of-subgaussian}
Let $T:\Z\to\X$ be a Lipschitz continuous map from $\Z\subset\R^k$ to $\X\subset\R^d$ with Lipschitz constant $L$ and such that $T(0) = 0$. Let $\eta\in\mathcal{G}_\sigma(\Z)$. Then $T_\push\eta\in\mathcal{G}_{\sigma_L}(\X)$ with $\sigma_L = L\sqrt{k/d}$.
\end{lemma}

\begin{proof}
The result follows by observing that, for any $\sigma_0$ we have
\eqals{
    \int_\X e^{\frac{\nor{x}^2}{2d\sigma_0^2} }~d(T_\push\eta)(x) & = \int_\Z e^{\frac{\nor{T(z)}^2}{2d\sigma_0^2} }~d\eta(z) \leq \int_\Z e^{\frac{L^2\nor{z}^2}{2d\sigma_0^2} }~d\eta(z).
}
Choosing $\sigma_0 = L\sqrt{k/d}$ yields the required upper bound. 
\end{proof}

\begin{lemma}[Convolution of two sub-Gaussian measures]
\label{lem:product-of-two-subgaussian}
Let $\sigma_1,\sigma_2>0$, $\mu\in\mathcal{G}_{\sigma_1}(\X)$ and $\rho\in\mathcal{G}_{\sigma_2}(\X)$. Then $\mu\ast\rho\in\mathcal{G}_{2\bar\sigma}$ with $\bar\sigma =  \max(\sigma_1,\sigma_2)$.
\end{lemma}

\begin{proof}
For any $\sigma>0$ we have 
\eqals{
    \int e^{\frac{\nor{x}^2}{2d\sigma^2} } ~d(\mu\ast\rho)(x) & = \int e^{\frac{\nor{y+w}^2}{2d\sigma^2} } ~d\mu(y)d\rho(w)\\
    & \leq \int e^{\frac{\nor{y}^2+\nor{w}^2}{ d\sigma^2} } ~d\mu(y)d\rho(w)\\
    & \leq \Bigg(\int e^{\frac{2\nor{y}^2}{ d\sigma^2} } ~d\mu(y)\Bigg)^{1/2} \Bigg(\int e^{\frac{2\nor{w}^2}{ d\sigma^2} } ~d\rho(w)\Bigg)^{1/2}.
}
Now, if $\sigma\geq 2\sigma_1$, we have
\eqals{
    \int e^{\frac{2\nor{y}^2}{ d\sigma^2} } ~d\mu(y) & \leq \int e^{\frac{\nor{y}^2}{ 2d\sigma_1^2} } ~d\mu(y)\leq 2.
}
Analogously for $\sigma\geq 2\sigma_2$. Therefore by taking $\sigma = 2\bar\sigma$ with $\bar\sigma = \max(\sigma_1,\sigma_2)$, we have 
\eqals{
\int e^{\frac{\nor{x}^2}{2d\sigma^2} } ~d(\mu\ast\rho)(x) \leq 2
}
as required.
\end{proof}

\begin{lemma}[Perturbation]
\label{lemma:perturbation}
Let $\mu\in\mathcal{G}_\sigma$ with $\sigma\geq1$. Let $\Phi_\sig\in\mathcal{G}_\sig$ for $0\leq\sig\leq\sigma$. Then, for any $\nu\in\mathcal{G}_{2\sigma}$, we have
\eqals{
    \abs{\sink(\nu,\Phi_\sig\ast\mu) - \sink(\nu,\mu)} \leq \msf{c}_1(\sigma,d) ~ \sig
}
with
\eqals{
    \msf{c}_1(\sigma,d) = 4d^{3/2}\sigma (1 + C_{1,d}4\sigma^2(1+2\sigma))~ \sig
}
and $C_{1,d}$ a constant depending only on the ambient dimension $d$.
\end{lemma}

\begin{proof}
Since $\delta\leq\sigma$, by applying \cref{lem:product-of-two-subgaussian} we have $\Phi_\sig\ast\mu\in\mathcal{G}_{2\sigma}$ and $\mu\in\mathcal{G}_{\sigma}\subset\mathcal{G}_{2\sigma}$. Therefore, we apply \cref{lem:formal-uniform-upper-bound-space-of-functions} to control
\eqals{
    \abs{\sink(\nu,\Phi_\delta\ast\mu) - \sink(\nu,\mu)} & \leq \sup_{f\in\F_{2\sigma}}~ \int f(x+w)~d\mu(x)d\Phi_\delta(w) - \int f(x)~d\mu(x)\\
    & = \sup_{f\in\F_{2\sigma}}~ \int \big(f(x+w) -  f(x)\big)~d\Phi_\delta(w)d\mu(x)\label{eq:perturbation-inequaltiy-intermediate}.
}
Note that for any $x,w\in\X$ we can define $H:[0,1]\to\R$ such that for any $t\in[0,1]$
\eqals{
    H(t) = f(x+tw).
}
Then, by the fundamental theorem of calculus we have
\eqals{
    \int_0^1 H'(t) ~dt & = H(1)-H(0) \\
    & = f(x) - f(x+w).
}
Now 
\eqals{
    H'(t) = \scal{\nabla f(x+tw)}{w},
}
which implies
\eqals{
    \abs{f(x+w)-f(x)} & \leq \int_0^1 \nor{\nabla f(x+tw)}\nor{w}~ dt\\
    & \leq \sqrt{d}\nor{w}\int_0^1 \nor{\nabla f(x+tw)}_\infty.
}
Now, by direct application of \cite[Prop. $1$]{mena2019statistical} for the functions in $\F_{2\sigma}$, we have
\eqals{
    \abs{D^1 f(x)} \leq \begin{cases} \nor{x} + C_{1,d}4\sigma^2(1+2\sigma) & \textrm{if} \quad \nor{x} \leq \sqrt{d}\sigma\\
    \nor{x}(1 + C_{1,d}2^{3/2}\sigma^{3/2}(1+(2\sigma)^{1/2})) & \textrm{otherwise.}
    \end{cases}
}
Therefore, since $\sigma>1$, 
\eqals{
    \nor{\nabla f(x)}_\infty \leq \msf{c}_0(\sigma,d)(1 + \nor{x}) \qquad \textrm{with} \qquad \msf{c}_0(\sigma,d) =  1 + C_{1,d}4\sigma^2(1+2\sigma).
}
We can therefore bound 
\eqals{
    \int_0^1\norm{\nabla f(x+tw)}~dt & \leq \msf{c}_0(\sigma,d)\int_0^1 \nor{x + tw}~dt\\
    & \leq \msf{c}_0(\sigma,d) (\nor{x} + \nor{w}).
}
Therefore, for any $x,w\in\X$, \eqals{
    \abs{f(x+w)-f(x)} \leq \sqrt{d}\msf{c}_0s(\sigma,d)(\nor{x}\nor{w} + \nor{w}^2).
}
Plugging this inequality in \cref{eq:perturbation-inequaltiy-intermediate}, we have
\eqals{
    \sup_{f\in\F_{\sigma}}~ \int f(x+w) -  f(x)~d\mu(x) & \leq \sqrt{d}\msf{c}_0(\sigma,d)\int \nor{x}\nor{w} + \nor{w}^2~d\mu(x)d\Phi_\sig(w)\\
    & \leq 2d^{3/2}\msf{c}_0(\sigma,d)~ \sig (\sigma + \sig)
}
where we have used \cref{lem:bound_on_moment_subg} in the last inequality. Since $\delta\leq\sigma$, the above inequality yields the required result. 
\end{proof}

We are ready to prove our result on perturbed GAN models

\begin{corollary}[Formal version of \cref{cor:perturbation}]\label{cor:formal-perturbation}
Under the same assumption of \cref{thm:main-rates}, let $\Phi_\sig\in\mathcal{G}_\delta(\X)$ and $\rho = \Phi_\sig\ast T^*_\push\eta^*$. Let $\msf c$ be the same constant of \cref{thm:main-rates}. Then,
\eqals{
    \mathbb{E}~\sink(\hat T_\push \hat \eta, \rho) ~\leq~ \frac{2\,\msf c(\tau,L,\sigma,k)}{\sqrt{n}} ~+~ 3\msf{c}_1(L\sigma\sqrt{k/d},d)~\delta.
}
with $\msf{c}_1(L\sigma\sqrt{k/d},d)$ the constant defined in \cref{lemma:perturbation}. 
\end{corollary}

\begin{proof}
Let $\rho_n$ be the empirical sample used to obtain $(\hat T,\hat\eta)$. By definition of $\rho$, we have that $\rho_n$ corresponds to an empirical sample of points $(T^*(z_i)+w_i)_{i=1}^n$ with $z_i$ i.i.d. points sampled from $\eta^*$ and $w_i$ i.i.d. points sampled from $\Phi_\sig$. We denote $\eta^*_n = \frac{1}{n}\sum_{i=1}^n \delta_{z_i}$. 

We start by considering the following decomposition of the error
\eqals{
    \sink(\hat T_\push \hat \eta, \rho) &  = A_1 + A_2 + A_3 + A_4 + A_5 + A_6
}
with
\eqals{
    A_1 & = \sink(\hat T_\push \hat \eta, \rho) - \sink(\hat T_\push \hat \eta,T^*_\push \eta^*)\\
    A_2 & = \sink(\hat T_\push \hat \eta,T^*_\push \eta^*) - \sink(\hat T_\push \hat \eta,T^*_\push \eta^*_n)\\
    A_3 & = \sink(\hat T_\push \hat \eta,T^*_\push \eta^*_n) - \sink(\hat T_\push \hat \eta,\rho_n)\\
    A_4 & = \sink(\hat T_\push \hat \eta,\rho_n) - \sink( T^*_\push \eta^*,\rho_n)\\
    A_5 & = \sink( T^*_\push \eta^*,\rho_n) - \sink( T^*_\push \eta^*,T^*_\push \eta^*_n)\\
    A_6 & = \sink(T^*_\push \eta^*,T^*_\push\eta^*_n)
}
We start by controlling the term $A_1$. First we note that according to \cref{lem:pushforward-of-subgaussian}, both distributions $\hat T_\push \hat \eta$ and $T^*_\push\eta^*$ are sub-Gaussian with parameter $ L\sigma\sqrt{k/d}$. Therefore, by applying \cref{lemma:perturbation} we obtain 
\eqals{
    A_1 \leq \msf{c}_1(L\sigma\sqrt{k/d},d)~\sig
}
where $\msf{c}_1$ is the constant introduced in \cref{lemma:perturbation}. 
We note that by adopting the analogous reasoning we can bound the terms $A_3$ and $A_5$. Namely, by taking the expectation with respect to the sample $\rho_n$
\eqals{
    \mathbb{E}[A_3] & \leq \msf{c}_1(L\sigma\sqrt{k/d},d)~\sig\\
    \mathbb{E}[A_5] & \leq \msf{c}_1(L\sigma\sqrt{k/d},d)~\sig.
}
The term $A_2$ corresponds to the sample complexity of $T^*_\push\eta^*$. Therefore we can apply \cref{thm:sample_complexity} to obtain 
\eqals{
    \mathbb{E}[A_2] & = \mathbb{E}~\Big[\sink(\hat T_\push \hat \eta,T^*_\push \eta^*) - \sink(\hat T_\push \hat \eta,T^*_\push \eta^*_n)\Big] \leq \frac{\mathsf{c}(\sigma,\tau,L)}{\sqrt{n}}.
}
The same holds for $A_6$, namely
\eqals{
    \mathbb{E}[A_6] & = \mathbb{E}~\Big[ \sink(T^*_\push \eta^*,T^*_\push\eta^*_n)\Big]\\
    & = \mathbb{E}~\Big[ \sink(T^*_\push \eta^*,T^*_\push\eta^*_n) - \sink(T^*_\push \eta^*,T^*_\push\eta^*)\Big]\\
    & \leq \frac{c(\sigma,\tau,L)}{\sqrt{n}}.
}
Finally, we note that since $(\hat T,\hat \eta)$ is the minimizer of $\sink(T_\push\eta,\rho_n)$, 
\eqals{
A_4 & = \sink(\hat T_\push \hat \eta,\rho_n) - \sink( T^*_\push \eta^*,\rho_n) \leq 0.
}
Combining all the bounds above yields the required result.
\end{proof}

\section{Optimization}
\label{app:gradient}


\subsection{Computing the gradient with respect to the network parameters}

In this section we provide the analytic formula for the gradient of the sinkhorn divergence with respect to the generator network parameters. We recall here the statament.

\PGeneratorGradient*

\begin{proof}
We prove a more general version of \cref{prop:generator-gradient}, replacing the squared-Euclidean cost in \cref{eq:primal_pb} (or equivalently \cref{eq:dual_pb}) with a generic smooth cost function $c:\X\times\X\to\R$. The proof of the result hinges upon the following characterization of the directional derivative of functionals that admit a variational form.

\begin{theorem}[Thm. $4.13$ in \cite{bonnans2013perturbation}]\label{thm:sensitivity-analysis}
Suppose that for all $x\in\X$ the function $f(x,\cdot)$ is (Gateaux) differentiable, that $f(x,u)$ and $\nabla_u f(x, u)$ are continuous on $X \times U$, and that the
inf-compactness condition holds. Then the optimal value function 
\eqals{
    v: u_0 \mapsto \inf_{x\in\X} f(x,u_0),
}
is Fr\'echet directionally differentiable at $u_0$ with directional derivative
\eqals{
    v'(u_0;\bar u) = \inf_{x\in{\cal{S}}(u_0)} \nabla_u f(x,u_0) \bar u,
}
for any $\bar u\in U$, with ${\cal{S}}(u_0)$ the set of minimizer of $f(\cdot,u_0)$.
\end{theorem}
We recall that inf-compactness is the condition requiring the existence of a neighborhood of $u_0$ and a constant ${T_\theta}_\push\eta\in\R$ such that the level sets of $f(\cdot,u)$ are compact for any $u$ in such neighborhood. We note that the same result holds when considering the supremum of a joint function $f(x,u_0)$, which is the case of the Sinkhorn divergence considered in the following.

Let now $\eta\in\P(\Z), \rho\in\P(\X)$ and $\Theta$ a space of parameters for the pushforward maps $T_\theta:\Z\to\X$. We will apply \cref{thm:sensitivity-analysis} to the functional
\eqals{
    F(\theta) = \oteps({T_\theta}_\push\eta,\rho) = \sup_{u,v\in \cont(\X)} G({T_\theta}_\push \eta, \rho,  u, v),
}
where we have denoted 
\eqals{
   G({T_\theta}_\push \eta, \rho,  u, v) = \int u(x)~d({T_\theta}_\push\eta)(x) & + \int v(y)~d\rho(y) \\
    & - \eps \int e^{\frac{u(x) + v(y)  - \cost(x,y)}{\eps}}~d({T_\theta}_\push\eta)(x)d\rho(y).
}
We recall that the solution $(u^*,v^*)$ to the Sinkhorn dual problem is unique up to a constant shift $(u^*+r,v^*-r)$ for any $r\in\R$ \cite{feydy2018interpolating}. We can therefore restrict the above optimization problem to 
\eqals{\label{eq:dual-problem-strictly-concave}
F(\theta) & = \sup_{(a,b)\in \mathcal{D} }G({T_\theta}_\push \eta, \rho,  u, v),
}
to the domain
\eqals{
    \mathcal{D} = \Bigg\{~(u,v)\in \cont(\X)\times \cont(\X) ~ \Bigg| ~ \int u(x)~d({T_\theta}_\push\eta)(x) = \int v(y)~d\rho(y) \Bigg\}.
}
Therefore, over this linear subspace of $\cont(\X)\times \cont(\X)$, the functional $\oteps({T_\theta}_\push\eta,\rho,\cdot,\cdot)$ admits a unique minimizer and is actually strictly concave, which guarantees inf-compactness (actually sup-compactness in this case) to hold.

We can therefore apply \cref{thm:sensitivity-analysis} with the following substitutions in our setting: $x\gets (u,v)$, $\X\gets\mathcal{D}$, $u\gets \theta$ and $U\gets\Theta$. Let $(u^*,v^*)$ be the minimizer of \cref{eq:dual-problem-strictly-concave}. We have
\eqals{
    [\nabla_\theta F(\cdot)]|_{\theta=\theta_0} & = \nabla_\theta [ G({T_\theta}_\push\eta,\rho,u^*,v^*)|_{\theta=\theta_0}.
}
Now, by applying the Transfer lemma, we have 
\eqals{
G({T_\theta}_\push\eta,\rho,u^*,v^*) = G(\eta,\rho,u^*\circ{T_\theta},v^*)
}
and therefore,
\eqals{
    [\nabla_\theta F(\cdot)]|_{\theta=\theta_0} & = \nabla_\theta [ G(\eta,\rho,u^*\circ T_\theta,v^*)|_{\theta=\theta_0} \\ 
    & = \int \big[\nabla_\theta u^*(T_\theta(z))\big]|_{\theta=\theta_0}~d\eta(z) \\
    & ~~~ -\eps \int \Bigg[\nabla_\theta e^{\frac{u^*(T_\theta(z)) + v^*(y)- \cost(T(z),y)}{\eps}}\Bigg]|_{\theta=\theta_0}~d\eta(z)d\rho(y)
}
Recall that $u^*$ is differentiable (actually $C^{\infty}$, see e.g. \cite[Thm. 2]{genevay2018sample} characterizing the regularity of Sinkhorn potential when using a smooth cost). Therefore, by applying the chain rule we have 
\eqals{\label{eq:chain-rule-first-piece}
    \int \big[\nabla_\theta u^*(T_\theta(z))\big]|_{\theta=\theta_0}~d\eta(z) & = \int \big[\nabla_x u^*(\cdot)|_{x=T_{\theta_0}(z)}\big]\big[\nabla_\theta T_\theta(z)\big]|_{\theta=\theta_0}~d\eta(z).
}
By applying computing the gradient of the exponetial term, the second term in the gradient can be split in two parts
\eqals{
 \eps\int \big[\nabla_\theta ~e^{\frac{u^*(T_\theta(z)) + v^*(y) - c(T_\theta(z),y)}{\eps}}\big]|_{\theta=\theta_0} & ~d\eta(z)d\rho(y) = A_1 - A_2
}
with
\eqals{
A_1 & = \int \big[\nabla_\theta u^*(T_\theta(z))\big]|_{\theta=\theta_0} ~e^{\frac{u^*(T_{\theta_0}(z)) + v^*(y) - c(T_{\theta_0}(z),y)}{\eps}}~d\eta(x)d\rho(y) \\
A_2 & = \int \big[\nabla_\theta c(T_\theta(z),y)\big]|_{\theta=\theta_0} ~e^{\frac{u^*(T_{\theta_0}(z)) + v^*(y) - c(T_{\theta_0}(z),y)}{\eps}}~d\eta(x)d\rho(y).
}
Recall that since $(u^*,v^*)$ is a pair of minimizers, the characterization of $u^*$ from \cref{eq:characterization-sinkhorn-potential} holds, implying that for any $z\in\Z$,
\eqals{
    \int e^{\frac{u^*(T_{\theta_0}(z)) + v^*(y) - c(T_{\theta_0}(z),y)}{\eps}}~d\rho(y) = 1
}
Therefore
\eqals{
    A_1 & = \int \big[\nabla_\theta u^*(T_\theta(z))\big]|_{\theta=\theta_0} ~\Bigg(\int e^{\frac{u^*(T_{\theta_0}(z)) + v^*(y) - c(T_{\theta_0}(z),y)}{\eps}}~d\rho(y)\Bigg)~d\eta(z)\\
    & = \int \big[\nabla_\theta u^*(T_\theta(z))\big]|_{\theta=\theta_0}~d\eta(z).
}
Hence, analogously to \cref{eq:chain-rule-first-piece} we have
\eqals{
    A_1 & = \int \big[\nabla_x u^*(\cdot)|_{x=T_{\theta_0}(z)}\big]\big[\nabla_\theta T_\theta(z)\big]|_{\theta=\theta_0}~d\eta(z).
}
Regarding the term $A_2$, we apply the chain rule to the cost term, obtaining
\eqals{
    A_2 & = \int \big[\nabla_x c(\cdot,y)|_{x=T_{\theta_0}(z)}\big]\big[\nabla_\theta T_\theta(z)\big]|_{\theta=\theta_0} ~e^{\frac{u^*(T_{\theta_0}(z)) + v^*(y) - c(T_{\theta_0}(z),y)}{\eps}}~d\eta(z)d\rho(y)
}
Since the term in \cref{eq:chain-rule-first-piece} and $A_1$ eliminate each other, we have
\eqals{
    [\nabla_\theta F(\cdot)]|_{\theta=\theta_0} = \int \big[\nabla_x c(\cdot,y)|_{x=T_{\theta_0}(z)}\big]\big[\nabla_\theta T_\theta(z)\big]|_{\theta=\theta_0} ~e^{\frac{u^*(T_\theta(z)) + v^*(y) - c(T_\theta(z),y)}{\eps}}~d\eta(z)d\rho(y)
}
Now, by the characterization of Sinkhorn potential in \cref{eq:characterization-sinkhorn-potential}, we have that for any $x_0\in\X$,
\eqals{
\nabla_x u^* (\cdot)|_{x=x_0} = \int\nabla_x c(\cdot,y)|_{x=x_0}~ e^{\frac{u^*(x_0)+v(y) - c(x_0,y))}{\eps}}d\rho(y).
}
Replacing the equality above in the characterization of $\nabla_\theta F$, we have
\eqals{
    \big[\nabla_\theta F(\cdot)\big]|_{\theta=\theta_0} = \int \big[\nabla_z u^*(\cdot)\big]|_{x=T_{\theta_0}(z)}~ \big[\nabla_\theta T_\theta(z)\big]|_{\theta=\theta_0} ~d\eta(z),
}
as required.
\end{proof}

\paragraph{Gradient of the Sinkhorn Divergence} \cref{prop:generator-gradient} characterizes the gradient of $\oteps({T_\theta}_\push\eta,\rho)$ with respect to the network parameters $\theta$. However, the Sinkhorn divergence, defined in \cref{eq:sink_divergence} depends also on the so-called {\itshape autocorrelation} term $-\frac{1}{2}\oteps({T_\theta}_\push\eta,{T_\theta}_\push\eta)$. By following the same reasoning in the proof of \cref{prop:generator-gradient}, we have that 
\eqals{
    \big[\nabla_\theta \oteps({T_\theta}_\push\eta,{T_\theta}_\push\eta)\big]|_{\theta=\theta_0} = 2\int \big[\nabla_z u^*(\cdot)\big]|_{x=T_{\theta_0}(z)}~ \big[\nabla_\theta T_\theta(z)\big]|_{\theta=\theta_0} ~d\eta(z)
}
with $u^*$ the Sinkhorn potential minimizing 
\eqals{
    \oteps({T_\theta}_\push\eta,{T_\theta}_\push\eta) = \sup_{u\in \cont(\X)}~G({T_\theta}_\push\eta,{T_\theta}_\push\eta,u,u).
}
We refer to \cite{feydy2018interpolating} for more details on the properties of the autocorrelation term above. 

Thanks to the linearity of the gradient, we can therefore compute the gradient of the Sinkhorn divergence by combining the gradient of $\oteps({T_\theta}_\push\eta,\rho)$ and $\oteps({T_\theta}_\push\eta,{T_\theta}_\push\eta)$.

\section{Experiments}\label{app:sec_experiments}
In this section we provide the details on the experimental setup of section \cref{sec:experiments}.

\paragraph{Spiral} We describe the setting reported in \cref{fig:spiral}, where the target $\rho\in\P(\R^2)$ is a multi-modal distribution on a spiral-shaped $1$D manifold in $\R^2$. In particular we modeled $\rho^* = T_\push\eta^*$ with $\eta$ a mixture of three Gaussian distributions on $\R$,
\eqals{
    \eta^* = \frac{1}{3}\sum_{j=1}^3 \mathcal{N}(m_j,\sigma)
}
with means respectively in $m_1 = 0.1$, $m_2 = 0.7$ and $m_3=0.9$ and same variance $\sigma^2 = 0.1$. To map $\eta$ to $\R^2$ we considered the pushforward map $T^*:\R\to\R^2$ such that 
\eqals{
    T^*(x) = (x\sin(2\pi x),~ x \cos(2\pi x)).
}
To approximate $T$ we considered a fully connected neural network with $4$ hidden layers with dimensions $256, 1024, 256, 256$, two ReLUs activation functions for the first and second layers and one sigmoid ($\tanh$) for the third layer. To minimize $\sink(T_\push\eta,\rho_n)$ in $T$ for $\eta$ fixed, we used ADAM as optimizer, with learning rate of $\alpha_1 = 10^{-4}$. To learn $\eta$ for $T$ fixed we used step size $\alpha_2 = 10^{-3}$. We run \cref{alg:latent-GAN} with $m=1000$ particles and sampling size $\ell=100$ (the number of ``perturbation'' points smapled around the particles at each iteration). We set the regularization parameter of the Sinkhorn divergence equal to $\eps=0.005$. When keeping $\eta$ fixed, we chose $\eta$ to be $\mathcal{N}(0.5,1)$. 

\paragraph{Swiss Roll}
Similary to the previous setting we considered $\rho = T^*_\push\eta^*$ the pushforward measure of a multimodal latent distribution. Here $\rho\in\P(\R^3)$, with $\eta^*\in\P(\R^2)$ the ``restriction'' of a Gaussian mixture on $[0,1]^2$. More formally, let $g:\R^2\to\R$ be the density of the mixture of $3$ isotropic Gaussian measures
\eqals{
    \pi = \frac{1}{3}\sum_{j=1}^3 \mathcal{N}(m_j, \Sigma)
}
with means $m_1 = (0.4,0.4)$, $m_2 = (0.2,0.8)$ and $m_3 = (0.8,0.5)$, and same covariance $\Sigma = \sigma I$ with $\sigma^2 = 0.15$. Then we consider $\eta^*$ the distribution with density $h:\R^2\to\R$, proportional to
\eqals{
    h \propto g\cdot \mathbbm{1}_{[0,1]^2},
}
where $\mathbbm{1}_{[0,1]^2}$ is the indicator function of the interval $[0,1]^2$. We used the pushforward map of the swiss roll $T^*:\R^2\to\R^3$ 
\eqals{
    T(x,y) ~=~ (~x\cos(2\pi x), y, x\sin(2\pi x)~).
}
To approximate $T$ we considered the same structure used in the spiral setting: a fully connected neural network with $4$ hidden layers with dimensions $256, 1024, 256, 256$, two ReLUs activation functions for the first and second layers and one sigmoid ($\tanh$) for the third layer. To minimize $\sink(T_\push\eta,\rho_n)$ in $T$ for $\eta$ fixed, we used ADAM as optimizer, with learning rate of $\alpha_1 = 5 \cdot 10^{-5}$. To learn $\eta$ for $T$ fixed we used step size $\alpha_2 = 10^{-4}$. We run \cref{alg:latent-GAN} with $m=1000$ particles and sampling size $\ell=100$ performing block-coordinate descent, alternating $50$ iterations when learning $T$ with $\eta$ fixed and $20$ iterations when learning $\eta$ for a fixed $T$. We set the regularization parameter of the Sinkhorn divergence starting from $\eps=2$ and decreasing every $50$ iterations of the generator training (both for \cref{alg:latent-GAN} and the standard Sinkhorn GAN) by a factor $\eps\gets 0.9 \cdot \eps$. We did not allow $\eps$ to drop below $10^{-3}$ When keeping the latent fixed, we chose $\eta$ to be $\mathcal{N}([0.5,0.5], I)$.  

\end{document}